\def\draft{0}
\def\llncs{0}

\ifnum\llncs=1
\documentclass[runningheads,a4paper]{llncs}
\usepackage{url,enumerate,latexsym,cite}
\usepackage{macrosLLNCS}
\else
\documentclass[11pt]{article}
\usepackage{macros}
\newif\ifsubmission
\submissiontrue
\fi

\ifnum\llncs=1
\index{Gupte, Aparna}
\index{Vafa, Neekon}
\index{Vaikuntanathan, Vinod}
\fi

\title{Sparse Linear Regression and Lattice Problems}

\author{
\ifnum\llncs=1
\author{Aparna Gupte\orcidID{0009-0007-2809-1257} \and Neekon Vafa\orcidID{0000-0002-0555-4200} \and Vinod Vaikuntanathan\orcidID{0000-0002-2666-0045}}
\institute{MIT CSAIL}
\else
Aparna Gupte\thanks{Research supported by the Ida M.~Green MIT Office of Graduate Education Fellowship.} \\MIT\\\texttt{agupte@mit.edu} \and Neekon Vafa\thanks{Research supported by NSF fellowship DGE-2141064 and by the grants of the third author.}\\MIT\\\texttt{nvafa@mit.edu} \and Vinod Vaikuntanathan\thanks{Research supported in part by DARPA under Agreement No. HR00112020023, NSF CNS-2154149, a grant from the MIT-IBM Watson AI, a grant from Analog Devices, a Microsoft Trustworthy AI grant,  a Thornton Family Faculty Research Innovation Fellowship from MIT and a Simons Investigator Award. Any opinions, findings and conclusions or recommendations expressed in this material are those of the author(s) and do not necessarily reflect the views of the United States Government or DARPA.}\\MIT\\\texttt{vinodv@mit.edu}
\fi
}
\date{}

\def\BDD{\mathsf{BDD}}
\def\BinBDD{\mathsf{BinaryBDD}}
\def\SLR{\mathsf{SLR}}
\def\sparse{\mathsf{sparse}}
\def\partite{\mathsf{partite}}

\def\CLWE{\mathsf{CLWE}}

\usepackage{bm}
\newcommand{\vect}[1]{\boldsymbol{\mathbf{#1}}}
\newcommand{\mat}[1]{\boldsymbol{\mathbf{#1}}}

\def\zero{\mathbf{0}}
\def\ones{\mathbf{1}}

\def\Z{\mathbb{Z}}

\def\R{\mathbb{R}}

\def\S{\mathcal{S}}

\def\cone{\mathbb{C}}
\def\sphere{\mathbb{S}}

\def\bin{\mathsf{bin}}
\def\lambdabin{\lambda_{1, \bin}}

\def\round{\mathsf{round}}
\def\col{\mathsf{col}}

\def\slrrows{m}
\def\bdddim{d}
\def\slrcols{n}
\def\REfac{\epsilon}
\def\RE{\zeta}
\def\cond{\kappa}
\def\TL{\mathsf{TL}}
\def\gamclwe{\gamma_{\CLWE}}
\def\ZWJ{\mathsf{ZWJ}}

\begin{document}
\maketitle

\begin{abstract}
Sparse linear regression (SLR) is a well-studied problem in statistics where one is given a design matrix $\mathbf{X} \in \mathbb{R}^{m \times n}$ and a response vector $\mathbf{y} = \mathbf{X} \boldsymbol{\theta}^* + \mathbf{w}$  for a $k$-sparse vector $\boldsymbol{\theta}^*$ (that is, $\|\boldsymbol{\theta}^*\|_0 \leq k$) and small, arbitrary noise $\mathbf{w}$, and the goal is to find a $k$-sparse $\widehat{\boldsymbol{\theta}} \in \mathbb{R}^{n}$ that minimizes the mean squared prediction error $\frac{1}{m} \|\mathbf{X} \widehat{\boldsymbol{\theta}} - \mathbf{X} \boldsymbol{\theta}^*\|^2_2$. While $\ell_1$-relaxation methods such as basis pursuit, Lasso, and the Dantzig selector solve SLR when the design matrix is well-conditioned, no general algorithm is known, nor is there any formal evidence of hardness in an {\em average-case} setting with respect to {\em all} efficient algorithms.

We give evidence of average-case hardness of SLR w.r.t.~all efficient algorithms assuming the worst-case hardness of lattice problems. Specifically, we give an {\em instance-by-instance} reduction from a variant of the bounded distance decoding (BDD) problem on lattices to SLR, where the condition number of the lattice basis that defines the BDD instance is directly related to the restricted eigenvalue condition of the design matrix, which characterizes some of the classical statistical-computational gaps for sparse linear regression. Also, by appealing to worst-case to average-case reductions from the world of lattices, this shows hardness for a \emph{distribution} of SLR instances; while the design matrices are ill-conditioned, the resulting SLR instances are in the identifiable regime. 

Furthermore, for well-conditioned (essentially) \emph{isotropic} Gaussian design matrices, where Lasso is known to behave well in the identifiable regime, we show hardness of outputting \emph{any} good solution in the \emph{unidentifiable} regime where there are many solutions, assuming the worst-case hardness of standard and well-studied lattice problems.
\end{abstract}

\ifnum\llncs=0
\newpage 
\tableofcontents
\newpage
\fi

\section{Introduction}
\label{sec:intro}

We study the fundamental statistical problem of {\em sparse linear regression} where one is given a design matrix $\vec X \in \mathbb{R}^{\slrrows\times \slrcols}$ and responses $\vec y \in \mathbb{R}^{\slrrows}$ where 
$$\vec y = \vec X \btheta^* + \vec w$$ 
for a hidden parameter vector $\btheta^* \in \mathbb{R}^\slrcols$ which is $k$-sparse, i.e., it has at most $k$ non-zero entries, and a small, arbitrary noise $\vec w \in \R^m$. The goal is to output a $k$-sparse vector $\widehat{\btheta}$ such that the {\em (mean squared) prediction error} $$\frac{1}{m} \|\matX\widehat{\btheta} - \matX \btheta^*\|_2^2$$ 
is as small as possible. Each row of $\vec X$ corresponds to a {\em sample} or a {\em measurement}, and each column of $\vec X$ corresponds to a {\em feature} of the model $\btheta^*$.

Information-theoretically, it is possible to achieve prediction error 
\[  \frac{1}{m} \norm{ \mat{X} \vect{\widehat{\theta}} - \mat{X} \vect{\theta}^*}^2_2 \leq \frac{4 \norm{\vect{w}}_2^2}{m}, \]
regardless of the design matrix $\mat{X}$. 
Algorithmically, however, the prediction error achieved by many efficient polynomial-time algorithms (such as $\ell_1$-relaxation or $\ell_1$-regularization including basis pursuit and Lasso estimators~\cite{tibshirani,CDS98} as well as the Dantzig selector~\cite{CT07}) depends on the conditioning of the design matrix $\matX$---specifically, the \emph{restricted-eigenvalue} constant  $\RE(\mat{X})$, which essentially lower bounds the smallest singular value of $\matX$ restricted to nearly sparse vectors. 
(See Definition~\ref{def:re-constant} for a formal definition.) By adapting the analysis of~\cite{negahban2012unified} (see Theorems~\ref{thm:general-lasso-bound} and \ref{thm:lasso-performance-in-terms-of-noise}), the thresholded Lasso estimator $\widehat{\vect{\theta}}_{\mathsf{TL}}$ achieves prediction error\footnote{This bound assumes $\mat{X}$ satisfies a certain column-normalization condition (see Definition~\ref{def:column-normalization}).}
\begin{align*}
    \delta_{\mathsf{Lasso}}^2 := \frac{1}{m} \| \mat{X} \widehat{\vect{\theta}}_{\mathsf{TL}} - \mat{X} \vect{\theta}^* \|_2^2 \le O\left(\frac{\|\vect{w}\|^2_2 \cdot k^2}{\RE(\mat{X})^2 \cdot m} \right).
\end{align*}
In particular, when the design matrix $\vec X$ is well-conditioned, i.e., $\RE(\mat{X}) = \Omega(1)$ which, for example, happens when the rows are drawn from $\N(\mathbf{0},\vec I_{\slrcols\times \slrcols})$ for $m = \Omega(k \log n)$, the $\ell_1$-relaxation or $\ell_1$-regularization algorithms achieve the information-theoretically optimal prediction error (up to polynomial factors in $k$). The smaller the restricted eigenvalue constant, the worse the prediction error bound. We emphasize that without computational bounds, the achievable prediction error does not depend on the characteristics of the design matrix $\vec X$.

In several naturally occurring high-dimensional regression problems, the design matrix $\matX$ may be ill-conditioned as the features, or even the samples, could be correlated. An important question then is to understand {\em which design matrices admit efficient sparse linear regression algorithms and which do not}. Stated differently, is Lasso (and friends) the best possible algorithm for sparse linear regression? This is the central question of interest in this paper.

A handful of works have started to explore this question both from the algorithmic front and the hardness front. On the algorithmic front, the recent work of~\cite{kelnerpreconditioning} showed an algorithm called {\em pre-conditioned Lasso} which achieves low prediction error for certain ill-conditioned matrices where Lasso provably fails. On the hardness front,~\cite{kelner2022lower, kelnerpreconditioning} prove hardness against particular algorithms, namely preconditioned Lasso. In terms of hardness against all efficient algorithms, \cite{zhang2014lower} construct a fixed design matrix $\mat{X}_{\ZWJ} \in \mathbb{R}^{m\times n}$ 
such that solving the $k$-sparse linear regression problem with better prediction error than Lasso for $k \approx n/4$ 
on an arbitrary (worst-case) $k$-sparse ground truth $\boldsymbol{\theta}^*$ implies that \textbf{NP} $\subseteq$ \textbf{P/poly}.\footnote{They show slightly more, i.e., they construct a distribution over $\boldsymbol{\theta}^*$ for which sparse linear regression is hard; however, to the best of our knowledge, this distribution is not polynomial-time sampleable. Indeed, if this distribution were sampleable, their result would show a worst-case to average-case reduction for \textbf{NP} which remains a major open problem in complexity theory.}
While an exciting initial foray into the landscape of hardness results for sparse linear regression, \cite{zhang2014lower} inspires many more questions:

\begin{enumerate}
    \item\label{zwj-question-family-of-instances} Most importantly, which design matrices $\vec X$ are hard for sparse linear regression? While \cite{zhang2014lower} gives us an example in the form of {\em a single} $\mat{X}_{\ZWJ}$, it does not give us much insight into the hardness profile of a given design matrix.
    \item\label{zwj-question-sparsity-parameter}The work of \cite{zhang2014lower} shows that finding $k$-sparse solutions for $k \approx n/4$ (i.e., constant factor sparsity) is hard, but is it significantly easier to find much sparser solutions, e.g. what happens with polynomial sparsity, logarithmic sparsity or even constant sparsity? The problem certainly becomes easier, but can we nevertheless show evidence of hardness?
    \item\label{zwj-question-fine-grained} A related question is that of fine-grained hardness: while \cite{zhang2014lower} shows evidence against polynomial-time algorithms, could there be non-trivial algorithms that solve sparse linear regression significantly faster than brute force search, i.e., in time $\slrcols^{o(k)}$?
    \end{enumerate}

\subsection{Our Results}
In this paper, we show hardness results for sparse linear regression that address all of these questions. In a nutshell, and hiding some details, we show an instance-by-instance reduction from (a variant of) the {\em bounded distance decoding problem} on a lattice with a given basis $\vec B$ to sparse linear regression w.r.t. a design matrix $\vec X$ that is essentially drawn from the Gaussian distribution $\N(\mathbf{0},\boldsymbol{\Sigma})$, where $\boldsymbol{\Sigma} = \boldsymbol{\Sigma}(\vec B)$ is determined by $\vec B$. 
In particular, this allows us to start from conjectured hard instances of lattice problems and construct (nearly) Gaussian design matrices (with a covariance matrix related to the lattice basis) for which $k$-SLR is hard, partially addressing Question~\ref{zwj-question-family-of-instances}. Our main theorem (Theorem~\ref{informalthm:main-bin-bdd-to-slr}; see also Remark~\ref{remark:fine-grained}) addresses Questions \ref{zwj-question-sparsity-parameter}, handling a large range of sparsity parameters $k$, and \ref{zwj-question-fine-grained}, by showing a fine-grained reduction.

As a secondary result, even for well-conditioned (essentially) \emph{isotropic} Gaussian design matrices, where Lasso is known to behave well in the identifiable regime, we show hardness of outputting \emph{any} good solution in the less standard \emph{unidentifiable} regime where there are many solutions, assuming the worst-case hardness of standard and well-studied lattice problems.

\paragraph{Binary Bounded Distance Decoding.} 
Our source of hardness is lattice problems. Given a matrix $\vec B \in \mathbb{R}^{d \times d}$, the lattice generated by $\vec B$ is $$ \mathcal{L}(\vec B) = \{ \vec B \vec z: \vec z \in \mathbb{Z}^d\}.$$
A lattice has many possible bases: indeed, $\vec B \vec U$ is a basis of $\mathcal{L}(\vec B)$ whenever $\vec U$ is a unimodular matrix, i.e., an integer matrix with determinant $\pm 1$. The minimum distance $\lambda_1(\vec B)$ of the lattice $\mathcal{L}(\vec B)$ is the (Euclidean) length of the shortest non-zero vector in $\mathcal{L}(\vec B)$. Note that $\lambda_1(\vec B)$ does not depend on $\vec B$, only on $\mathcal{L}(\vec B)$.\footnote{We actually use a slightly different definition of $\lambda_1(\vec B)$ which is unimportant for this exposition; we refer the reader to Section~\ref{sec:prelims} for more details.} 

A canonical lattice problem is the bounded distance decoding (BDD): given a basis $\vec B$ of a lattice, a target vector $\vec t \in \mathbb{R}^d$, and a parameter $\alpha < 1/2$, find a lattice vector $\vec v \in \mathcal{L}(\vec B)$ such that $\mathsf{dist}(\vec t, \vec v) \leq \alpha \cdot \lambda_1(\vec B)$ given the promise that such a vector exists (equivalently, that $\vec t$ is sufficiently close to the lattice). By the bound on $\alpha$, there is a unique such vector. The BDD problem has been very well studied, especially in the last decade, and is widely believed to be hard, for instance, forming the basis of a new generation of post-quantum cryptographic algorithms recently standardized by \cite{NIST}. The best algorithm known for BDD runs in time $2^{d + o(d)}$~\cite{aggarwal2015solving}.\footnote{Aggarwal et al.~\cite{aggarwal2015solving} give an algorithm for solving the exact Closest Vector Problem (CVP), which immediately implies an algorithm for BDD.}

Our hardness assumption is a variant called {\em binary BDD} (see, e.g. \cite{DBLP:conf/crypto/KirchnerF15}) which asks for a vector $\vec v \in \vec B\cdot \{\pm 1\}^d$ 
that is close to the target $\vec t$ (as above), again under the promise that such a vector exists. Equivalently, given a vector $\vec t = \vec B \vec z + \vec e$ where $\vec z \in \{\pm 1\}^d$ and $\norm{\vect{e}} \leq \alpha\cdot \lambda_1(\vec B)$, the problem is to find $\vec z$. Both BDD and binary BDD are believed to be hard for arbitrary lattice bases $\vec B$. Indeed, a canonical way to solve (binary as well as regular) BDD is to employ Babai's rounding algorithm~\cite{babai1986lovasz} which works as follows: compute and output  $$\round(\vec B^{-1} \vec t) = \round(\vec z + \vec B^{-1} \vec e) = \vec z + \round(\vec B^{-1} \vec e)~,$$ namely, round each coordinate of $\vec B^{-1} \vec t$ to the nearest integer. This works as long as each coordinate of $\vec B^{-1} \vec e$ is at most $1/2$, which happens as long as $\vec e$ is small {\em and} $\vec B$ has a good condition number. In particular, Babai succeeds if $\alpha = O(1/\kappa(\vec B))$, where $\kappa(\matB) = \sigma_{\max}(\matB)/\sigma_{\min}(\matB)$ is the condition number of $\matB$. 
In other words, the condition number of $\vec B$ determines the performance of the algorithm.

To be sure, there are algorithms for {\em binary} BDD that perform slightly better than BDD: in particular, \cite{DBLP:conf/crypto/KirchnerF15} show a $2^{O(d/\log \log(1/\alpha))}$-time algorithm for binary BDD for a large range of $\alpha$ (see 
\ifnum\llncs=0
Theorem~\ref{thm:kirchner-fouque-alg}).
\else
Theorem~7 of the full version~\cite{gupte2024sparse}).
\fi
In particular, for $\alpha = 1/\poly(d)$, this becomes a $2^{O(d / \log \log d)}$-time algorithm for binary BDD, whereas the best run-time for BDD algorithms in this regime is $2^{\Theta(d)}$.
More than that, \cite{DBLP:conf/crypto/KirchnerF15} study binary BDD (and generalizations) in detail and give evidence of hardness: they show reductions from variants of $\mathsf{GapSVP}$ and $\mathsf{UniqueSVP}$ to (a slight generalization) of binary BDD, as well as a direct reduction from low-density subset sum to binary BDD \cite[Theorem 14]{DBLP:conf/crypto/KirchnerF15}. In fact, \cite{DBLP:conf/crypto/KirchnerF15} use this reduction and their $2^{O(d/\log \log(1/\alpha))}$-time algorithm for binary BDD to give a state-of-the-art algorithm for low-density subset sum. Improving on this $2^{O(d/\log \log(1/\alpha))}$ run-time bound for binary BDD would consequently give better algorithms for low-density subset sum.\footnote{We remark that the binary LWE (learning with errors) problem, where the LWE secret is binary, is {\em not} a special case of binary BDD even though LWE is a special case of BDD. Indeed, when writing a binary LWE instance as a BDD instance in the canonical way, only a part of the coefficient vector of the closest lattice point is binary. Thus, even though binary LWE is equivalent in hardness to LWE,~\cite{DBLP:conf/innovations/GoldwasserKPV10, DBLP:conf/stoc/BrakerskiLPRS13,micciancio2018hardness}, we do not know such a statement relating binary BDD and BDD that preserves conditioning of the lattice basis.}

\paragraph{Our First Result.} 
Our first result shows that for every lattice basis $\vec B$, there is a related covariance matrix $\boldsymbol{\Sigma} = \boldsymbol{\Sigma}(\vec B)$ such that if binary BDD is hard given the basis $\vec B$, then $k$-SLR is hard w.r.t. an {\em essentially} Gaussian design matrix whose rows are i.i.d. $\N(\mathbf{0},\boldsymbol{\Sigma})$.\footnote{The bottom $k\times n$ sub-matrix of $\matX$ (call it $\matX_2$) is a fixed, worst-case, matrix while each of the top rows is drawn i.i.d. from $\N(\zero, \boldsymbol{\Sigma})$.
There are two natural ways to remove the ``worst-case'' nature of $\mat{X}_2$.
We discuss them in Section~\ref{sec:removing-worst-case-gadget}.} The more precise statement follows. For $\beta \in [0,2]$, we say that a $k$-$\SLR$ algorithm is a \emph{$\beta$-improvement} of Lasso if on input $(\mat{X} \in \R^{m \times n}, \vect{y} \in \R^m)$, the algorithm achieves prediction error $\delta^2$ where
\[ \delta^2 = \delta_{\mathsf{Lasso}}^2 \cdot \RE(\mat{X})^{\beta} \cdot \poly(k, \log n) \le \frac{\|\vect{w}\|^2_2}{\RE(\mat{X})^{2 - \beta} \cdot m}  \cdot \poly(k, \log n).\]
Note that $\beta = 0$ corresponds to Lasso itself, and $\beta = 2$ achieves the information-theoretic bound (up to $\poly(k, \log n)$ factors).

\begin{theorem}\label{informalthm:main-bin-bdd-to-slr}
There is a $\poly(m, k \cdot 2^{d/k})$-time randomized reduction \sloppy{from $\BinBDD$} in $d$ dimensions with parameter $\alpha \leq 1/10$ to $k$-$\SLR$ in dimension $\slrcols = k \cdot 2^{\bdddim/k}$ and $m \geq 17 d$ samples that succeeds with probability $1 - e^{-\Omega(m)}$. Moreover, the reduction maps a $\BinBDD$ instance w.r.t. lattice basis $\mat{B}$ to design matrices $\mat{X}\in \mathbb{R}^{m\times n}$ where 
\begin{itemize}
    \item the distribution of each of the top $m-k$ rows is i.i.d. $\N(\zero, \mat{\Sigma})$ where $\mat{\Sigma} = \mat{G}_{\sparse}^\top \mat{B}^\top \mat{B} \mat{G}_{\sparse}$ for a fixed, instance-independent matrix $\mat{G}_{\sparse} \in \R^{d \times n}$; 
    \item each of the bottom $k$ rows is proportional to a fixed, instance-independent vector that depends only on $\slrcols$ and $k$; and 
    \item if there is a $k$-$\SLR$ polynomial-time algorithm that is a $\beta$-improvement to Lasso, then there is a $\poly(m, k \cdot 2^{d/k})$-time algorithm for $\BinBDD$ in $d$ dimensions with parameter $$\alpha \leq  \frac{1}{ \poly(d) \cdot \kappa(\vec B)^{2 - \beta}}$$ 
    w.r.t. lattice bases $\mat{B}$.
\end{itemize}
\end{theorem}

The first qualitative take-away message from our theorem is: if you beat Lasso\footnote{More precisely, $\beta$-improve Lasso for $\beta > 1$.}, you beat Babai. The performance of the plain Lasso algorithm is determined by the restricted eigenvalue constant; analogously, the performance of Babai's rounding algorithm is bounded by the condition number of the basis matrix.  Our theorem says that if you come up with an algorithm that beats the Lasso guarantee, in the sense of achieving prediction error $\frac{1}{\RE(\matX)^{2-\beta}}$ for $\beta > 1$,  then you have at hand an algorithm that beats the BDD approximation factor achieved by Babai's algorithm by a corresponding amount, namely solve BDD to within a factor of $\frac{1}{\kappa(\vec B)^{2-\beta}}$, whereas Babai's algorithm itself achieves an approximation factor of $\frac{1}{\kappa(\matB)}$. When $\beta = 2$, the $k$-SLR algorithm achieves the information-theoretic optimal prediction error (up to polynomial factors in $d$ and $k$) in which case it gives us a polynomial time binary BDD algorithm for $\alpha$ that is inverse-polynomial in $d$.

To be sure, there are (recent) algorithms that solve sparse linear regression beating the RE constant bound, e.g. \cite{kelnerpreconditioning,DBLP:journals/corr/abs-2305-16892}. These algorithms work by first {\em pre-conditioning} the design matrix and running the plain Lasso w.r.t. the preconditioned matrix. Our reduction then says that for BDD basis matrices which map to the easy $k$-SLR instances identified by \cite{kelnerpreconditioning,DBLP:journals/corr/abs-2305-16892}, you {\em can} beat Babai. However, that should {\em not} be surprising, in general: there {\em are} basis matrices $\matB$ which can be ``pre-conditioned'', e.g. using the Lenstra-Lenstra-Lov\'{a}sz basis reduction algorithm or its variants~\cite{lenstra1982factoring,DBLP:journals/tcs/Schnorr87}, thereby improving their condition number and consequently the performance of Babai. Indeed, in our view, understanding the relationship between the two types of preconditioning transformations --- one from the $k$-SLR world~\cite{kelnerpreconditioning,DBLP:journals/corr/abs-2305-16892} and the other from the lattice world~\cite{lenstra1982factoring,DBLP:journals/tcs/Schnorr87} --- is a fascinating open question.

On the other hand, there are also basis matrices whose condition number cannot be improved in polynomial time. (If not, then running such a conditioning algorithm and then Babai would give efficient worst-case lattice algorithms, which we believe do not exist.) Indeed, given what we know about the hardness of worst-case lattice problems, our theorem identifies a large class of design matrices where solving sparse linear regression is at least as hard.

Furthermore, our theorem shows the hardness of sparse linear regression for a range of sparsity parameters, addressing Question~\ref{zwj-question-sparsity-parameter}. The sparser the instance, the longer the run-time of the binary BDD algorithm guaranteed by the reduction. On one extreme, when $k = \Omega(d/\log d)$, the reduction runs in $\mathsf{poly}(d)$ time and the sparsity $k=n^{1 - \epsilon}$ can be set to any polynomial function of $n$. On the other extreme, when $k$ is sufficiently super-constant and (say) $\beta = 2$, we get slightly subexponential-time (in $d$) algorithms for binary BDD, beating the algorithm of \cite{DBLP:conf/crypto/KirchnerF15} for $k = \omega(\log \log d)$. Our reduction also addresses Question~\ref{zwj-question-fine-grained}, because gives us fine-grained hardness of $k$-$\SLR$ algorithms running in time $n^{o(k)}$ (see Remark~\ref{remark:fine-grained}).

In order to directly address our Question~\ref{zwj-question-family-of-instances}, we can combine our main reduction with known lattice reductions to generate an average-case hard instance for $k$-SLR. In particular, we can reduce the average-case learning with errors problem~\cite{DBLP:journals/jacm/Regev09}, to (standard) BDD with bounded norm solutions, and then to binary BDD using a gadget from \cite{DBLP:conf/eurocrypt/MicciancioP12} (%
\ifnum\llncs=0
see Appendix~\ref{sec:reduction-from-bdd-to-binbdd}
\else
see full version~\cite{gupte2024sparse}%
\fi
). In turn, since the learning with errors problem has a reduction from worst-case lattice problems~\cite{DBLP:journals/jacm/Regev09,DBLP:conf/stoc/Peikert09}, this hardness result relies only on \emph{worst-case} hardness assumptions, despite being \emph{average-case} over $k$-SLR instances. While the resulting distribution over sparse linear regression design matrices is degenerate and ill-conditioned, the resulting $k$-SLR instance is still in the identifiable regime and information-theoretically solvable. As far as the authors are aware, this is the first \emph{average-case} hardness result over $k$-SLR instances, and relies only on \emph{worst-case} lattice assumptions.

Finally, as a teaser to our techniques, we emphasize that lattice problems are fundamentally about integrality, while sparse-linear regression has no integrality constraint. To make our reduction go through, we convert an \emph{integrality} constraint in the lattice problem into a \emph{sparsity} constraint by enforcing an additional linear constraint via a certain gadget matrix. See Section~\ref{sec:proof-main-reduction-and-gadgets} for more details. We add that \cite{har2016approximate} previously used essentially the same gadget matrix in their reduction from $k$-SUM to $k$-SLR.

\paragraph{Our Second Result.}
Our second result, Theorem~\ref{thm:reduction-from-clwe}, shows the  hardness of sparse linear regression with a  design matrix  $\vec X$ whose rows are  essentially drawn from the spherical Gaussian $\N(\mathbf{0},\vec I_{n\times n})$, as long as the number of samples is smaller than $k\log d$. Even though this puts us in the non-identifiable regime where there may be many $k$-sparse $\widehat{\theta}$ consistent with $(\vec X, \vec y)$, the problem of minimizing the mean squared error\footnote{In the unidentifiable regime, minimizing the prediction error is not information-theoretically possible without some constraint on the noise.} is well-defined. More precisely, the top $m-k$ rows of $\vec X$ are i.i.d. $\N(\mathbf{0},\vec I_{n\times n})$ and the bottom $k$ rows are the same fixed matrix from Theorem~\ref{informalthm:main-bin-bdd-to-slr}. This is a regime where Lasso would have solved the problem with a larger number of samples, giving some evidence of the optimality of Lasso in terms of sample complexity.
The hardness assumption in this case is the continuous learning with errors (CLWE) assumption which was recently shown to be as hard as standard and long-studied worst-case short vector problems on lattices~\cite{BRST20,gupte2022continuous}, which are used as the basis for many post-quantum cryptographic primitives.
We state the result formally in Theorem~\ref{thm:reduction-from-clwe} and prove it in Section~\ref{sec:reduction-from-clwe}.

\subsection{Perspectives and Open Problems}

We view our results as an initial foray into understanding the landscape of {\em average-case hardness} of the sparse linear regression problem. One of our results shows a distribution of average-case hard $k$-SLR instances (under worst-case hardness of lattice problems). However, the design matrices that arise in this distribution are very ill-conditioned; even though the $k$-SLR instances they define are identifiable, the design matrices are singular and have RE constant $0$. An immediate open question arising from this result is whether one can come up with a more robust average-case hard distribution for $k$-SLR. 

More broadly, the connection to lattice problems that we exploit in our reduction inspires a fascinating array of open questions. To begin with, as we briefly discussed above, the recent improvements to $k$-SLR~\cite{kelnerpreconditioning,DBLP:journals/corr/abs-2305-16892} proceed via pre-conditioning the design matrix. Analogously, the famed LLL algorithm of \cite{lenstra1982factoring} from the world of lattices is nothing but a pre-conditioning algorithm for lattice bases.
This leads us to ask: is there a constructive way to use (ideas from) lattice basis reduction to get improved $k$-SLR algorithms? We mention here the results of \cite{DBLP:journals/tit/GamarnikKZ21,DBLP:conf/colt/ZadikSWB22} for positive indications of the effectiveness of lattice basis reduction in solving statistical problems. In the lattice world, it is known that if one allows for arbitrary (potentially unbounded-time) pre-processing of a given lattice basis, BDD can be solved with a polynomial approximation factor~\cite{DBLP:conf/approx/LiuLM06}. Is a similar statement true for $k$-SLR?  We believe an exploration of these questions and others is a fruitful avenue for future research.

\paragraph{Concurrent work.} 
The concurrent work of \cite{Buhai24} shows hardness of achieving non-trivial prediction error in sparse linear regression with Gaussian designs in the improper setting (where the estimator need not be sparse) by reducing from a slight variant of a standard sparse PCA problem. In their setting, they deduce a sample complexity lower bound of (roughly) $k^2$ for efficient algorithms, where (roughly) $k$ is possible information-theoretically, by assuming a (roughly) $k^2$ sample complexity lower bound on efficient algorithms for the sparse PCA problem. (Their reduction produces instances in which the noise is standard Gaussian. Our main theorem is written in terms of worst-case noise, but we show a similar statement for independent Gaussian noise in %
\ifnum\llncs=0
Appendix~\ref{section:independent-noise}.)%
\else
the full version~\cite{gupte2024sparse}.)%
\fi

\subsection{Road Map of Main Results}

We prove Theorem~\ref{informalthm:main-bin-bdd-to-slr} through a sequence of steps:
\begin{enumerate}
    \item In Section~\ref{sec:main-reduction}, we state and prove our reduction from $\BinBDD_{d, \alpha}$ to $k$-$\SLR$ (Theorem~\ref{thm:main-bin-bdd-to-slr}).
    \item In Section~\ref{sec:proof-of-re-constant-appendix}, we give a bound on the restricted-eigenvalue constant $\RE$ of the instances produced by this reduction in terms of $\cond(\matB)$, the condition number of the $\BinBDD_{d, \alpha}$ lattice instance $\matB$ (Lemma~\ref{lemma:re-constant-lower-bound}).
    \item In Section~\ref{sec:lasso-on-our-instances-appendix}, using our reduction and the guarantees of Lasso in terms of $\RE$ (Theorem~\ref{thm:lasso-performance-in-terms-of-noise}), we state how quantitative \emph{improvements} to Lasso (in terms of the dependence on $\RE$) give algorithms for $\BinBDD_{d, \alpha}$ with quantitatively stronger parameters (Theorem~\ref{thm:better-lasso-gives-better-bdd}).
\end{enumerate}
For the second result (hardness of $k$-$\SLR$ in the non-identifiable regime even for \emph{well-conditioned} design matrices), we give the proof in Section~\ref{sec:reduction-from-clwe}.

\section{Preliminaries}\label{sec:prelims}
For $n \in \mathbb{N}$, we use the notation $[n] := \{1, 2, \cdots, n\}$. We use the standard definitions of the $\ell_1, \ell_2$, and $\ell_{\infty}$ norms in $\R^n$, as well as the $\ell_0$ ``norm'' which is defined as the \emph{sparsity}, or number of non-zero entries of a vector $\vect{v} \in \R^n$. We use $\vect{bold}$ notation to denote vectors and matrices. For a matrix $\mat{A} \in \R^{m \times n}$ and $i \in [m]$, we write $\col_i(\mat{A}) \in \R^n$ to denote the $i$th column of $\mat{A}$. We let $I_{n \times n}$ denote the standard identity matrix in $\R^n$. For $\vect{v} \in \R^n$, we use the notation $\round(\vect{v}) \in \Z^n$ to denote the point-wise rounding function applied to $\vect{v} \in \R^n$ (with arbitrary behavior at half integers). We use the standard definitions of the maximum and minimum \emph{singular values} of $\mat{A}$ for $m \geq n$:

\[ \sigma_{\min}(\mat{A}) := \min_{\vect{v} \in \R^n} \frac{\norm{\mat{A} \vect{v}}_2}{\norm{\vect{v}}_2},\ \ \ \ \ \ \sigma_{\max}(\mat{A}) := \max_{\vect{v} \in \R^n} \frac{\norm{\mat{A} \vect{v}}_2}{\norm{\vect{v}}_2}. \]
We use the notation $\cond(\mat{A})$ to denote the \emph{condition number} of $\mat{A}$, defined as
\[ \cond(\mat{A}) := \frac{ \sigma_{\max}(\mat{A})}{\sigma_{\min}(\mat{A})}. \]
We use the notation $\N(\mu, \sigma^2)$ to denote the univariate normal distribution with mean $\mu$ and standard deviation $\sigma$. Similarly, we use the notation $\N(\vect{\mu}, \mat{\Sigma})$ to denote the multivariate normal distribution with mean $\vect{\mu} \in \R^n$ and positive semi-definite covariance matrix $\mat{\Sigma} \in \R^{n \times n}$. 
We emphasize that we do not require $\mat{\Sigma}$ to be positive definite. 

\subsection{Bounded Distance Decoding}
We now define lattice quantities that will be useful for us.

\begin{definition}
For a full-rank matrix $\mat{B} \in \R^{\bdddim \times \bdddim}$, the lattice $\mathcal{L}(\mat{B})$ generated by basis $\mat{B}$ consists of the set all integer linear combinations of the columns of $\mat{B}$. As is standard, we define 
\[ \lambda_1(\mat{B}) := \min_{\vect{z} \in \Z^{\bdddim} \setminus \{\zero\}} \norm{\mat{B} \vect{z}}_2,\]
which corresponds to the shortest distance between any two elements of $\mathcal{L}(\mat{B})$. In this paper, since we consider a ($\pm1$) \emph{binary} version of the bounded distance decoding problem, we define the quantity
\[ \lambdabin(\mat{B}) := \min_{\vect{z}_1 \neq \vect{z}_2 \in \{1, -1\}^{\bdddim}} \norm{\mat{B} \vect{z}_1 - \mat{B} \vect{z}_2}_2. \]
\end{definition}
Note that while $\lambda_1(\mat{B})$ is a basis-independent quantity for the lattice $\mathcal{L}(\mat{B})$, $\lambdabin(\mat{B})$ depends on the basis $\mat{B}$.
\begin{lemma}\label{lemma:lambda-1-bound}
For any matrix $\mat{B} \in \R^{\bdddim \times \bdddim}$, we have $\sigma_{\min}(\mat{B}) \leq \lambda_1(\mat{B}) \leq \lambdabin(\mat{B}) \leq 2\sigma_{\max}(\mat{B})$.
\end{lemma}
\begin{proof}
It is immediate that $\lambda_1(\mat{B}) \leq \lambdabin(\mat{B})$. Let $\vect{e} \in \Z^{\bdddim}$ be any standard basis vector. We have
\[\lambdabin(\mat{B}) = \min_{\vect{z}_1 \neq \vect{z}_2 \in \{1, -1\}^{\bdddim}} \norm{\mat{B} (\vect{z}_1 - \vect{z}_2)}_2 \leq 2\norm{\mat{B} \vect{e}}_2 \leq 2\sigma_{\max}(\mat{B}) \norm{\vect{e}}_2 = 2\sigma_{\max}(\mat{B}). \]
We also have
\[\lambda_1(\mat{B}) = \min_{\vect{z} \in \Z^{\bdddim} \setminus\{\zero\}} \norm{\mat{B} \vect{z}}_2 \geq  \min_{\vect{z} \in \Z^{\bdddim} \setminus\{\zero\}} \sigma_{\min}(\mat{B}) \norm{\vect{z}}_2 \geq \sigma_{\min}(\mat{B}), \]
since all $\vect{z} \in \Z^{\bdddim} \setminus \{\zero\}$ satisfy $\norm{\vect{z}}_2 \geq 1$.
\end{proof}

\begin{definition}[$\BinBDD_{d, \alpha}$]
Let $\BinBDD_{d, \alpha}$ be the following worst-case search problem in dimension $\bdddim \in \mathbb{N}$ with parameter $\alpha < 1/2$. Given full rank $\mat{B} \in \R^{\bdddim \times \bdddim}$ and $\vect{t} \in \R^{\bdddim}$, output $\vect{z} \in \{-1,1\}^{\bdddim}$ such that $\norm{\mat{B} \vect{z} - \vect{t}}_2 \leq \alpha \cdot \lambdabin(\mat{B})$ (if one exists). Equivalently, given $(\mat{B}, \mat{B} \vect{z} + \vect{e})$ for $\norm{\vect{e}}_2 \leq \alpha \cdot \lambdabin(\mat{B})$, output $\vect{z} \in \{-1,1\}^{\bdddim}$.
\end{definition}

Note that by definition of $\lambdabin(\mat{B})$, the constraint $\alpha < 1/2$ guarantees the uniqueness of $\vect{z} \in \{-1,1\}^{\bdddim}$. As an aside, the standard $\BDD$ problem instead uses $\lambda_1(\matB)$ and allows outputting any $\vect{z} \in \Z^d$ instead of $\vect{z} \in \{-1, 1\}^d$. Additionally, note that we could easily consider full-rank $\matB \in \R^{d_1 \times d}$ for $d \leq d_1 \leq O(d)$ and our main results would all go through, and similarly for larger $d_1$ with a slight change in the parameter dependence. Furthermore, even though we write $\mat{B} \in \R^{\bdddim \times \bdddim}$, we will assume the entries of $\mat{B}$ have $\poly(\bdddim)$ bits of precision. (This is just for simplicity and convenience; one can modify the corresponding algorithmic statements so that they run in polynomial time in the description length of the basis $\matB$.)

We note that there is a reduction from general BDD to $\BinBDD$, at the cost of making the instance degenerate (in the sense of having a non-trivial kernel). See %
\ifnum\llncs=0
Appendix~\ref{sec:reduction-from-bdd-to-binbdd}.
\else
the full version~\cite{gupte2024sparse}.
\fi

We now recall standard spectral bounds for random Gaussian matrices.
\begin{lemma}[As in \cite{rudelson2010non}]\label{lemma:gaussian-singular-values}
Let $\mat{R} \in \R^{\slrrows \times \bdddim}$ be such that $R_{i,j} \sim_{\iid} \N(0,1)$. For all $t > 0$, we have
\[ \Pr\left[\sqrt{\slrrows} - \sqrt{\bdddim} - t \leq \sigma_{\min}(\mat{R}) \leq \sigma_{\max}(\mat{R}) \leq \sqrt{\slrrows} + \sqrt{\bdddim} + t \right] \geq 1 - 2 e^{-t^2/2}.\]
In particular, for $\slrrows \geq 16\bdddim$, by setting $t = \sqrt{\slrrows}/4$, we have
\[\Pr\left[\frac{\sqrt{\slrrows}}{2} \leq \sigma_{\min}(\mat{R}) \leq \sigma_{\max}(\mat{R}) \leq \frac{3\sqrt{\slrrows}}{2} \right] \geq 1 - 2 e^{-\slrrows/32}.\]
\end{lemma}
\noindent Throughout the paper, let $\chi^2(m)$ denote the chi-squared distribution with $m$ degrees of freedom.

\begin{lemma}[As in \cite{laurent2000adaptive}, Corollary of Lemma~1]\label{lemma:chi-squared-tail-bound}
Let $X \sim \chi^2(\slrrows)$, i.e., $X$ is distributed according to the $\chi^2$ distribution with $\slrrows$ degrees of freedom. For any $t \geq 0$, we have
\[ \Pr[X \geq \slrrows + 2 \sqrt{t\slrrows} + 2t] \leq e^{-t}. \]
In particular, setting $t = \slrrows/4$, we get
\[ \Pr[X \geq 5\slrrows/2] \leq e^{-\slrrows/4}.\]
\end{lemma}

\subsection{Sparse Linear Regression}
We now define the problem of $k$-sparse linear regression ($k$-$\SLR$ for short), which, for computational complexity simplicity, is phrased in terms of mean squared error (as opposed to mean squared prediction error).

\begin{definition}[$k$-$\SLR$]\label{def:worst-case-k-slr}
Given a design matrix $\mat{X} \in \R^{\slrrows \times \slrcols}$, a target vector $\vect{y} \in \R^{\slrrows}$, and $\delta > 0$, output a $k$-sparse $\vect{\widehat{\theta}} \in \R^{\slrcols}$ such that
\[\frac{\|\mat{X} \vect{\widehat{\theta}} - \vect{y}\|^2_2}{\slrrows} \le \delta^2,\]
assuming one exists.
\end{definition}

We follow the convention of~\cite{negahban2012unified} of working with design matrices $\mat{X}$ that are \emph{column-normalized}:

\begin{definition}[As in \cite{negahban2012unified}]\label{def:column-normalization}
We say a matrix $\widetilde{\mat{X}} \in \R^{m \times n}$ is \emph{column-normalized} if for all $i \in [n]$, \[\norm*{\col_i (\widetilde{\mat{X}})}_2 \leq \sqrt{m}.\]
\end{definition}

Now, we show that a column-normalized $\widetilde{\mat{X}}$ also satisfies a spectral bound with respect to sparse vectors.

\begin{lemma}\label{lem:max-sparse-X-tilde}
Suppose $\widetilde{\mat{X}} \in \R^{m \times n}$ is column-normalized. Then, 
\begin{align*}
    \max_{\substack{\norm{\vect{v}}_2 = 1,\\ \norm{\vect{v}}_0 \le k}} \norm{\widetilde{\mat{X}}\vect{v}}_2 \le \sqrt{km}.
\end{align*}
\end{lemma}
\begin{proof}
Let $\vect{v}$ be a vector such that $\norm{\vect{v}}_2 = 1$ and $\norm{\vect{v}}_0 \le k$. By the triangle inequality and column normalization, we have
\begin{align*}
    \norm{ \widetilde{\mat{X}} \vect{v}}_2 = \norm*{ \sum_{i \in [n]} v_i \cdot  \col_{i}(\widetilde{\mat{X}}) }_2 \leq   \sum_{i \in [n]} |v_i| \cdot  \norm*{\col_{i}(\widetilde{\mat{X}}) }_2 \leq \sqrt{m} \cdot \norm{\vect{v}}_1 \leq \sqrt{km},
\end{align*}
where the last inequality holds because $\norm{\vect{v}}_1 \leq \sqrt{k} \norm{\vect{v}}_2 = \sqrt{k}$.
\end{proof}

To analyze the performance Lasso algorithm on the instances of $k$-$\SLR$ we obtain from our reduction, we slightly modify the analysis of~\cite{negahban2012unified} to allow for a more flexible definition of the \emph{restricted eigenvalue} (RE) constant of a matrix. First, we define the following cone. 
\begin{definition}
Let $S \subseteq [\slrcols]$ and $\REfac > 0$. We define $\cone_{\REfac}(S)$, the $\REfac$-cone for $S \subseteq [\slrcols]$, as follows:
\[\mathbb{C}_{\REfac}(S) := \{ \Delta \in \R^{\slrcols} : \|\Delta_{\bar{S}}\|_1 \le (1 + \REfac) \|\Delta_S\|_1 \}.\]
The notation $\bar{S}$ denotes the complement of $S$, namely $[\slrcols] \setminus S$, and the notation $\Delta_{I} \in \R^{|I|}$ denotes the restriction of $\Delta$ to coordinates in $I \subseteq [\slrcols]$.
\end{definition}
A standard definition (e.g., as in~\cite{negahban2012unified}) sets $\REfac = 2$, but we will use the flexibility of setting $\REfac$ close to $0$. Now, we define the restricted eigenvalue condition.

\begin{definition}\label{def:re-constant}
Suppose $\widetilde{\mat{X}} \in \R^{\slrrows \times \slrcols}$ is column-normalized. Then, we say $\widetilde{\mat{X}}$ satisfies the $(\REfac, \RE)$-\emph{restricted eigenvalue (RE) condition} for $S \subseteq [\slrcols]$ if
\[ \frac{\norm{\mat{\widetilde{X}} \vect{\theta}}_2^2}{\slrrows \cdot  \norm{\vect{\theta}}_2^2} \geq \RE\]
for all $\vect{\theta} \in \cone_{\REfac}(S)$.
\end{definition}
Note that this restricted eigenvalue condition corresponds to a restricted form of a strong convexity condition over the loss function
\[ \mathcal{L}(\vect{\theta}) := \frac{1}{2\slrrows} \norm{\vect{y} - \widetilde{\mat{X}} \vect{\theta}}_2^2. \]
Also note that the $\ell_1$-regularization-based Lasso estimator does not output a $k$-sparse solution. To obtain a $k$-sparse solution, we truncate $\vect{\widehat{\theta}}_\lambda$ to the $k$ entries of largest absolute value, to obtain the \textit{thresholded Lasso} estimate $\vect{\widehat{\theta}}_{\TL}$. Lemma~9 of \cite{zhang2014lower} shows that $\norm{\vect{\widehat{\theta}}_{\TL} - \vect{\theta}^*}_2 \le 5 \norm{\vect{\widehat{\theta}}_{\lambda} - \vect{\theta}^*}_2$.

Now, we state a modified version of Corollary 1 of~\cite{negahban2012unified}, which tells us an error bound on the thresholded Lasso estimator, given our modified definition of the restricted eigenvalue condition.

\begin{theorem}[Modified Version of Corollary 1 of \cite{negahban2012unified}]\label{thm:general-lasso-bound}
Let $\epsilon \in (0, 2]$, let $\widetilde{\mat{X}} \in \R^{\slrrows \times \slrcols}$ be column-normalized, and let $\widetilde{\vect{y}} = \widetilde{\mat{X}} \vect{\theta}^* + \widetilde{\vect{w}}$ for some $k$-sparse vector $\vect{\theta}^*$ supported on $S$ for $S \subseteq [\slrcols], |S| = k$. Suppose $\widetilde{\mat{X}}$ satisfies the $(\REfac, \RE)$-restricted eigenvalue condition for $S$. As long as 
\[ \lambda \geq \frac{2 + \REfac}{\REfac} \cdot \norm*{ \frac{1}{\slrrows} \widetilde{\mat{X}}^\top \widetilde{\vect{w}}}_{\infty}, \]
and $\lambda > 0$, then any Lasso solution $\vect{\widehat{\theta}}_{\lambda}$ with regularization parameter $\lambda$ satisfies the bound
\begin{align*}
    \norm{\vect{\widehat{\theta}}_{\lambda} - \vect{\theta}^*}^2_2 &\leq O \left( \frac{\lambda^2 k}{\RE^2 } \right).
\end{align*}
\end{theorem}

While the bounds in the theorem statement above do not depend on column normalization of $\widetilde{\mat{X}}$ (or corresponding normalization of $\widetilde{\vect{w}}$), our definition of the restricted eigenvalue condition needs the matrix to be column-normalized (in particular, so that the RE constant is scale invariant with respect to $\mat{B}$). The $\REfac$ dependence in $\lambda$ comes from modifying~\cite[Lemma 1]{negahban2012unified}, which needs the bound on $\lambda$ as above to guarantee that the optimal error $\widehat{\Delta} := \vect{\widehat{\theta}}_{\lambda} - \vect{\theta}^*$ satisfies $\widehat{\Delta} \in \cone_{\REfac}(S)$.

\begin{theorem}\label{thm:lasso-performance-in-terms-of-noise}
Let $\epsilon \in (0, 2]$, let $\widetilde{\mat{X}} \in \R^{\slrrows \times \slrcols}$ be column-normalized, and let $\widetilde{\vect{y}} = \widetilde{\mat{X}} \vect{\theta}^* + \widetilde{\vect{w}}$ for some $k$-sparse vector $\vect{\theta}^*$ supported on $S$ for $S \subseteq [\slrcols], |S| = k$. Suppose $\widetilde{\mat{X}}$ satisfies the $(\REfac, \RE)$-restricted eigenvalue condition for $S$. For an optimally chosen positive regularization parameter,
the thresholded Lasso solution $\vect{\widehat{\theta}}_{\TL}$ with satisfies the prediction error bound
\begin{align*}
    \frac{1}{m} \| \widetilde{\mat{X}} \widehat{\vect{\theta}}_{\mathsf{TL}} - \widetilde{\mat{X}} \vect{\theta}^*\|_2^2 \le O\left( \frac{\norm{\widetilde{\vect{w}}}_2^2 \cdot k^2}{\RE^2 \cdot  \epsilon^2 \cdot m} \right).
\end{align*}
\end{theorem}

\begin{proof}
    From Theorem~\ref{thm:general-lasso-bound}, we know that for
    \[ \lambda \geq \frac{2 + \REfac}{\REfac} \cdot \norm*{ \frac{1}{\slrrows} \widetilde{\mat{X}}^\top \widetilde{\vect{w}}}_{\infty}, \]
    we have
    \begin{align*}
    \norm{\vect{\widehat{\theta}}_{\lambda} - \vect{\theta}^*}_2 &\leq O \left( \frac{\lambda \sqrt{k}}{\RE } \right).
\end{align*}
By \cite[Lemma 9]{zhang2014lower}, we therefore have
\begin{align*}
    \norm{\vect{\widehat{\theta}}_{\TL} - \vect{\theta}^*}_2 &\leq O \left( \frac{\lambda \sqrt{k}}{\RE } \right).
\end{align*}
Plugging in the optimal choice of $\lambda$ yields
    \begin{align*}
    \norm{\vect{\widehat{\theta}}_{\TL} - \vect{\theta}^*}_2 &\leq O \left( \frac{\sqrt{k}}{\epsilon \cdot \RE } \cdot \norm*{ \frac{1}{\slrrows} \widetilde{\mat{X}}^\top \widetilde{\vect{w}}}_{\infty} \right).
\end{align*}
Since $\widetilde{\matX}$ is column-normalized, we know
\begin{align*}
    \norm*{ \frac{1}{\slrrows} \widetilde{\mat{X}}^\top \widetilde{\vect{w}}}_{\infty} = \frac{1}{m} \max_{i \in [n]} \left| \inner*{\col_i(\widetilde{\matX}), \widetilde{\vect{w}}} \right| \leq \frac{1}{m} \norm*{\col_i(\widetilde{\matX})}_2 \norm*{\widetilde{\vect{w}}}_2 \leq \frac{\norm*{\widetilde{\vect{w}}}_2}{\sqrt{m}},
\end{align*}
where the first inequality is due to Cauchy-Schwarz, and the second inequality is due to column normalization. Combining the two above inequalities, we have
    \begin{align*}
    \norm{\vect{\widehat{\theta}}_{\TL} - \vect{\theta}^*}_2 &\leq O \left( \frac{\sqrt{k} \cdot \norm*{\widetilde{\vect{w}}}_2}{\epsilon \cdot \RE \cdot \sqrt{m}} \right).
\end{align*}
Since $\vect{\widehat{\theta}}_{\TL} - \vect{\theta}^*$ is $2k$-sparse, we can apply Lemma~\ref{lem:max-sparse-X-tilde} to get
    \begin{align*}
    \norm{\widetilde{\matX} \vect{\widehat{\theta}}_{\TL} - \widetilde{\matX} \vect{\theta}^*}_2 &\leq O \left( \frac{k \cdot \norm*{\widetilde{\vect{w}}}_2}{\epsilon \cdot \RE} \right).
\end{align*}
Squaring and dividing by $m$ gives the desired result.
\end{proof}

\begin{remark}
We remark that Theorems~\ref{thm:general-lasso-bound} and \ref{thm:lasso-performance-in-terms-of-noise} imply that in the noiseless setting of $k$-$\SLR$, that is, when $\vect{y} = \mat{X} \thetastar$, Lasso not only achieves prediction error $0$ but also recovers the ground truth $\thetastar$ as long as the restricted eigenvalue constant is strictly positive. In particular, this means that any reduction showing the $\mathbf{NP}$-hardness of noiseless sparse linear regression must produce instances with design matrices with restricted eigenvalue $0$ (unless $\mathbf{P} = \mathbf{NP}$).
\end{remark}

However, there is no reason information theoretically that the prediction error should have dependence on the RE constant $\RE$. We now state a bound on the prediction error of the optimal $\ell_0$ predictor, which need not be computationally efficient. 

\begin{proposition}\label{prop:worst-case-info-theoretic-bound}
Let $\mat{X} \in \R^{m \times n}$, and let $\vect{y} = \mat{X} \vect{\theta}^* + \vect{w}$ for some $k$-sparse $\vect{\theta}^* \in \R^n$. The $\ell_0$-predictor
\[\vect{\widehat{\theta}} \in \argmin_{\norm{\vect{\theta}}_0 \leq k}  \norm{\mat{X} \vect{\theta} - \vect{y}}_2^2\]
satisfies the prediction error bound
\[  \frac{\norm{ \mat{X} \vect{\widehat{\theta}} - \mat{X} \vect{\theta}^*}^2_2}{m} \leq \frac{4 \norm{\vect{w}}_2^2}{m}. \]
\end{proposition}
\begin{proof}
By definition of $\vect{\widehat{\theta}}$ and since $\vect{\theta}^*$ is $k$-sparse, we know
\[ \norm{\mat{X}\vect{\widehat{\theta}} - \vect{y}}_2 \leq  \norm{\mat{X}\vect{\theta}^* - \vect{y}}_2 = \norm{\vect{w}}_2. \]
By the triangle inequality,
\[  \norm{\mat{X}\vect{\widehat{\theta}} - \mat{X}\vect{\theta}^*}_2 \leq \norm{\mat{X}\vect{\widehat{\theta}} - \vect{y}}_2 + \norm{\vect{y} - \mat{X}\vect{\theta}^*}_2 \leq 2 \norm{\vect{w}}_2.\]
The bound immediately follows by squaring and dividing by $m$.
\end{proof}

Comparing Theorem~\ref{thm:lasso-performance-in-terms-of-noise} and Proposition~\ref{prop:worst-case-info-theoretic-bound} highlights a critical computational-statistical gap: the prediction error for the Lasso algorithm depends quantitatively on the restricted eigenvalue condition of the design matrix, whereas the information-theoretic minimizer has prediction error that is completely independent of the restricted eigenvalue condition. 

\section{Reduction from Bounded Distance Decoding}\label{sec:main-reduction}

Our main result in this section is a reduction from $\BinBDD_{d, \alpha}$ to $k$-$\SLR$:

\begin{theorem}\label{thm:main-bin-bdd-to-slr}
Let $\bdddim, \slrrows, \slrcols, k$ be integers such that $k$ divides $\bdddim$, $\slrrows \ge 17\bdddim$ and $\slrcols = k \cdot 2^{\bdddim/k}$. Then, there is a $\poly(\slrcols, \slrrows)$-time reduction from $\BinBDD_{d, \alpha}$ in dimension $\bdddim$ with parameter $\alpha \leq 1/10$ to $k$-$\SLR$ in dimension $\slrcols$ with $\slrrows$ samples. The reduction is randomized and succeeds with probability $1 - e^{-\Omega(\slrrows)}$. Moreover, the reduction maps a $\BinBDD_{d, \alpha}$ instance $(\mat{B}, \cdot)$ to $k$-$\SLR$ instances with design matrices $\matX\in \R^{\slrrows \times \slrcols}$ and parameter $\delta = \Theta(\lambda_{1, \bin}(\matB))$ such that $\mat{X}$ that can be decomposed as
\[ \mat{X} = \begin{pmatrix}
\mat{X}_1\\
\mat{X}_2
\end{pmatrix}, \]
where the distribution of each row of $\mat{X}_1 \in \R^{(\slrrows-k) \times \slrcols}$ is distributed as $\iid$ $\N(\zero, \mat{G}_{\sparse}^\top \mat{B}^\top \mat{B} \mat{G}_{\sparse})$ for a fixed, instance-independent matrix $\mat{G}_{\sparse} \in \R^{d \times n}$, and $\mat{X}_2 \in \R^{k \times \slrcols}$ is proportional to a fixed, instance-independent matrix that depends only on $\slrcols$ and $k$. 
\end{theorem}

\begin{remark}\label{remark:fine-grained}
This reduction also gives us a \emph{fine-grained} hardness result. Suppose there is an algorithm for $k$-$\SLR$ in $n$ dimensions with $m$ samples that runs in time $\poly(m, k) \cdot (n/k)^{k^{1-\epsilon}}$. Then, the above reduction gives an algorithm for $\BinBDD_{d, \alpha}$ in $d$ dimensions that runs in time $\poly(d) \cdot 2^{d/k^\epsilon}$.

\cite{gupte2021finehardness} also prove a fine-grained hardness result for $k$-$\SLR$ from a lattice problem, specifically the closest vector problem. However, the key difference is that their result is in the unidentifiable regime, whereas ours is in the identifiable regime, i.e. the $k$-SLR instances produced by the reduction from binary BDD have a unique solution. Further, their fine-grained hardness results are for worst-case covariates, whereas our hardness results produces covariates drawn from a Gaussian (with a covariance matrix determined by the binary-BDD basis).
\end{remark}

\subsection{Interpretations}\label{sec:removing-worst-case-gadget}
We now interpret the distribution of the design matrices $\mat{X}$ in the reduction in Theorem~\ref{thm:main-bin-bdd-to-slr}. As stated, the sub-matrix $\mat{X}_2$ is phrased as a worst-case design matrix, while $\mat{X}_1$ has i.i.d. $\N(\zero, \mat{G}_{\sparse}^\top \mat{B}^\top \mat{B} \mat{G}_{\sparse})$ rows, where the covariance matrix is tightly related to the $\BinBDD_{d, \alpha}$ instance $\mat{B}$. (The definition of the fixed, instance-independent gadget matrix $\mat{G}_{\sparse}$ is given in Section~\ref{sec:proof-main-reduction-and-gadgets}.)

There are two natural ways to remove the ``worst-case'' nature of $\mat{X}_2$ in our instances:

\begin{enumerate}
    \item Our reduction still holds if the $\mat{X}_2$ part of our design matrix is removed, and instead, one enforced an exact linear constraint on $\vect{\theta}$ (specifically, $\mat{G}_{\partite} \vect{\theta} = \ones$, where $\mat{G}_{\partite}$ is defined in Section~\ref{sec:proof-main-reduction-and-gadgets}). This corresponds to constraining the $k$-sparse regression vector $\vect{\theta} \in \R^{\slrcols}$ to an affine subspace. Since many of the techniques for sparse linear regression involve convex optimization, intersecting the solution landscape with this affine subspace will preserve convexity, so these algorithms would still work if the design matrix were just $\mat{X}_1$ on its own, without the worst-case $\mat{X}_2$ sub-matrix.
    \item Alternatively, one can interpret each row of $\mat{X}_2$ as a mean of a Gaussian distribution with a covariance matrix $\sigma^2 I_{n \times n}$ for very small $\sigma > 0$. By reordering the rows of $\mat{X}$ and entries of the target $\vect{y}$, all rows of $\mat{X}$ can now be drawn i.i.d. from a \emph{mixture} of $k+1$ Gaussians, with weight $(m-k)/m$ on $\N(\zero, \mat{G}_{\sparse}^\top \mat{B}^\top \mat{B} \mat{G}_{\sparse})$ and weight $1/m$ on $\N(\vect{v}_i, \sigma^2 I_{n \times n})$ for all $i \in [k]$, where $\vect{v}_i$ is the $i$th row of $\mat{X}_2$ and $\sigma > 0$ is very small. Our reduction will go through (with modified parameters) as long as we increase the number of samples $m$ to roughly $O(m \log(k))$ so that each row of $\mat{X}_2$ is hit by a coupon collector argument.
\end{enumerate}

\subsection{Proof of Theorem~\ref{thm:main-bin-bdd-to-slr}}\label{sec:proof-main-reduction-and-gadgets}

Before proving Theorem~\ref{thm:main-bin-bdd-to-slr}, we introduce some notation and gadgets that will be useful for us. For $\slrcols, k \in \mathbb{N}$ where $\slrcols$ is a multiple of $k$, let $\S_{\slrcols, k}$ denote the set of $k$-sparse, $k$-partite binary vectors $\vect{v} \in \{0,1\}^{\slrcols}$. That is, for each $i \in \{0, \cdots, k-1\}$, there exists a unique $j \in [\slrcols/k]$ such that $v_{i\slrcols/k + j} = 1$, and all other entries of $\vect{v}$ are $0$.

Let $\mat{G}_{\sparse} \in \R^{\bdddim \times \slrcols}$ be the following (rectangular) block diagonal matrix, which will allow us to map binary vectors to binary sparse partite vectors. Each of the $k$ blocks $\mat{H} \in \R^{\bdddim/k \times \slrcols/k}$ are identical, with $\slrcols/k = 2^{\bdddim/k}$, and has columns consisting of all vectors $\{-1,1\}^{\bdddim/k}$, in some fixed order. 

\begin{lemma}\label{lemma:G-is-invertible}
The matrix $\mat{G}_{\sparse}$ invertibly maps $\S_{\slrcols, k}$ to $\{-1,1\}^{\bdddim}$.
\end{lemma}

\begin{proof}
Observe that $|\S_{\slrcols,k}| = (\slrcols/k)^k = 2^{\bdddim} = |\{-1,1\}^{\bdddim}|$, so it suffices to show that $\mat{G}_{\sparse}$ is surjective. Let $\vect{z} \in \{-1,1\}^{\bdddim}$. Breaking $\vect{z}$ up into $k$ contiguous blocks of $\bdddim/k$ coordinates, we can identify each block of $\vect{z}$ in $\{-1,1\}^{\bdddim/k}$ with a unique column in $\mat{H}$. Let $\vect{v} \in \S_{\slrcols, k}$ be defined so that for the $i$th block, $v_{ik + j} = 1$ if and only if the $j$th column of $\mat{H}$ is equal to the $i$th part of $\vect{z}$. It then follows that $\mat{G}_{\sparse} \vect{v} = \vect{z}$, as desired.
\end{proof}

We now define a different gadget matrix, $\mat{G}_{\partite} \in \R^{k \times \slrcols}$, that will help enforce the regression vector to be partite and binary. The matrix $\mat{G}_{\partite}$ is once again block diagonal with $k$ identical blocks, where each block is now the vector $\ones^{\top} \in \R^{1 \times \slrcols/k}$.

\begin{lemma}\label{lemma:G-forces-partite-binary}
For $\vect{v} \in \R^{\slrcols}$, we have $\vect{v} \in \S_{\slrcols,k}$ if and only if $\mat{G}_{\partite} \vect{v} = \ones$ and $\vect{v}$ is $k$-sparse.
\end{lemma}

\begin{proof}
For the ``only if'' direction, observe that $\vect{v} \in \S_{\slrcols, k}$ directly implies that $\vect{v}$ is $k$-sparse and $\mat{G}_{\partite} \vect{v} = \ones$, as each part of $\vect{v}$ sums to $1$.

For the ``if'' direction, suppose $\mat{G}_{\partite} \vect{v} = \ones$ and $\vect{v}$ is $k$-sparse. We need to show that $\vect{v}$ is both \emph{partite} and \emph{binary}. To see that $\vect{v}$ is partite, observe that the condition $\mat{G}_{\partite} \vect{v} = \ones$ enforces the sum of the entrices of $\vect{v}$ in each of the $k$ blocks to be $1$. Therefore, each block has at least one non-zero entry, as otherwise that block would sum to $0$. Since each of the $k$ blocks has at least one non-zero entry, by sparsity, each block has exactly one non-zero entry, as otherwise, $\vect{v}$ would not be $k$-sparse. Since each block has exactly one non-zero entry that sums to $1$, that entry must be $1$. Therefore, $\vect{v} \in \S_{\slrcols, k}$.
\end{proof}

Now, we prove the main theorem of this section, Theorem~\ref{thm:main-bin-bdd-to-slr}.
\begin{proof}[Proof of Theorem~\ref{thm:main-bin-bdd-to-slr}]
Let $(\mat{B}, \vect{t} := \mat{B}\vect{z} + \vect{e})$ be our $\BinBDD_{d, \alpha}$ instance, where $\mat{B} \in \R^{\bdddim \times \bdddim}$, $\vect{z} \in \{-1,1\}^{\bdddim}$ and $\norm{\vect{e}}_2 \leq \alpha \cdot  \lambdabin(\mat{B})$. Let $\mat{R} \in \R^{\slrrows_1 \times \bdddim}$ be a random matrix where each entry is drawn $\iid$ from $\N(0,1)$. We assume we know a value $\widehat{\lambda}_{1}$ such that
$\widehat{\lambda}_{1} \in [\lambdabin(\mat{B}), 2 \lambdabin(\mat{B}))$. (By Lemma~\ref{lemma:lambda-1-bound}, doubling search over $\widehat{\lambda}_1$ and taking the best $\BinBDD$ solution will take at most polynomial time.) 

Our design matrix $\mat{X}$ and target vector $\vect{y}$ for $k$-$\SLR$ will be defined as follows:
\[ \mat{X} = \begin{pmatrix}
\mat{X}_1\\
\mat{X}_2
\end{pmatrix} = \begin{pmatrix}
\mat{R} \mat{B} \mat{G}_{\sparse} \\
\gamma \mat{G}_{\partite}
\end{pmatrix} \in \R^{(\slrrows_1 + k) \times \slrcols},\ \ \ \ \vect{y} = \begin{pmatrix}
\mat{R} \vect{t}\\
\gamma \ones
\end{pmatrix} \in \R^{\slrrows_1 + k}.\]
Here, $\gamma \in \R$ is a scalar that will be set later. We define $\slrrows_1$ so that $\slrrows_1 + k = \slrrows$, which implies that $\slrrows_1 = \slrrows - k \geq \slrrows - \bdddim \geq 16\bdddim$. We define the scalar $\delta$ for the $k$-$\SLR$ instance so that
\begin{equation}\label{eqn:setting-delta}
\delta^2 = \frac{3 \slrrows_1 \cdot \widehat{\lambda}_1^2}{100\slrrows} = \Theta \left(\lambdabin(\mat{B})^2 \right).
\end{equation}
Invoking the $k$-$\SLR$ solver on the instance $(\mat{X}, \vect{y})$ with parameter $\delta$, our reduction will get a solution $\vect{\widehat{\theta}}$ such that
\begin{equation}\label{eq:squared-loss} \frac{\norm{\mat{X} \vect{\widehat{\theta}} - \vect{y}}_2^2}{\slrrows} \leq \delta^2.
\end{equation}
Next, we round $\vect{\widehat{\theta}}$ to the nearest integer entrywise to get $\round(\vect{\widehat{\theta}}) \in \Z^{\slrcols}$, and we output $\mat{G}_{\sparse} \round(\vect{\widehat{\theta}}) \in \Z^{\bdddim}$ as the $\BinBDD_{d, \alpha}$ solution.

First, we show completeness, in the sense that there exists $k$-sparse $\vect{\widehat{\theta}} \in \R^{\slrcols}$ such that~\eqref{eq:squared-loss} is satisfied. Recall from Lemma~\ref{lemma:G-is-invertible} that there exists a unique $\vect{\theta}^* \in \S_{\slrcols, k}$ such that $\mat{G}_{\sparse} \vect{\theta}^* = \vect{z}$. We will show that setting $\vect{\widehat{\theta}} = \vect{\theta}^*$ satisfies~\eqref{eq:squared-loss}. We have
\[\mat{X} \vect{\theta}^* - \vect{y} =   \begin{pmatrix}
\mat{R} \mat{B} \mat{G}_{\sparse} \\
\gamma \mat{G}_{\partite}
\end{pmatrix} \vect{\theta}^* - \begin{pmatrix}
\mat{R} \vect{t}\\
\gamma \ones
\end{pmatrix} = \begin{pmatrix}
\mat{R} \mat{B} \mat{G}_{\sparse} \vect{\theta}^* - \mat{R} \vect{t} \\
\gamma \mat{G}_{\partite} \vect{\theta}^* - \gamma \ones 
\end{pmatrix} = \begin{pmatrix}
\mat{R} \mat{B} \vect{z} - \mat{R} \vect{t} \\
\zero
\end{pmatrix},\]
where the last equality holds by Lemma~\ref{lemma:G-forces-partite-binary} and definition of $\vect{\theta}^*$. Continuing the equality, we have
\[\mat{X} \vect{\theta}^* - \vect{y} = \begin{pmatrix}
\mat{R} \mat{B} \vect{z} - \mat{R} (\mat{B} \vect{z} + \vect{e}) \\
\zero
\end{pmatrix} = \begin{pmatrix}
- \mat{R} \vect{e}\\
\zero
\end{pmatrix},\]
which implies $\norm{\mat{X} \vect{\widehat{\theta}} - \vect{y}}_2^2 = \norm{\mat{R} \vect{e}}_2^2$. Since $\mat{R}$ was sampled independently of $\vect{e}$, observe that $\norm{\mat{R} \vect{e}}_2^2 \sim \norm{\vect{e}}_2^2 \cdot \chi^2(\slrrows_1)$. By Lemma~\ref{lemma:chi-squared-tail-bound}, it follows that
\[ \Pr\left[\norm{\mat{R} \vect{e}}_2^2 \geq 5\slrrows_1/2 \cdot \norm{\vect{e}}_2^2 \right] \leq e^{-\slrrows_1/4}. \]
Therefore, with probability at least $1 - e^{-\slrrows_1/4}$, we have
\[ \frac{\norm{\mat{X} \vect{\widehat{\theta}} - \vect{y}}_2^2}{\slrrows} \leq \frac{5 \slrrows_1 \norm{\vect{e}}_2^2}{2\slrrows} \leq \frac{5 \slrrows_1 \alpha^2 \cdot \widehat{\lambda}_1^2}{2\slrrows} < \delta^2, \]
since $\alpha \leq 1/10$, as desired, showing that~\eqref{eq:squared-loss} is satisfied.

We now show soundness, which we will prove by contradiction. Suppose that $\mat{G}_{\sparse} \round(\vect{\widehat{\theta}}) = \vect{z}'$ for some $\vect{z}' \neq \vect{z}$ and~\eqref{eq:squared-loss} holds. Since~\eqref{eq:squared-loss} holds, in particular, we know $\norm{\gamma \mat{G}_{\partite} \vect{\widehat{\theta}} - \gamma \ones}_2^2 \leq \slrrows \delta^2$, or equivalently,
\[
\norm{\mat{G}_{\partite} \vect{\widehat{\theta}} - \ones}_2 \leq \sqrt{\slrrows} \delta/\gamma.
\]
Throughout, assume we set $\gamma$ so that $\sqrt{\slrrows} \delta / \gamma < 1/2$. Each part of $\vect{\widehat{\theta}}$ must have exactly one non-zero entry, as $\vect{\widehat{\theta}}$ is $k$-sparse and cannot have any part with all zero entries. Moreover, we know this non-zero entry must be in $[1 - \sqrt{\slrrows} \delta/\gamma, 1 + \sqrt{\slrrows} \delta/\gamma]$, so by a choice of $\gamma$ such that $\sqrt{\slrrows} \delta / \gamma < 1/2$, this entry must round to $1$. Therefore, $\round(\vect{\widehat{\theta}}) \in \S_{\slrcols, k}$, and $\round(\vect{\widehat{\theta}})$ and $\vect{\widehat{\theta}}$ have the same support. This implies that $\vect{z}' = \mat{G}_{\sparse} \round(\vect{\widehat{\theta}}) \in \{-1,1\}^{\bdddim}$ and $\mat{G}_{\partite} \round(\vect{\widehat{\theta}}) = \ones$ by Lemma~\ref{lemma:G-forces-partite-binary}. Moreover, since we can restrict to the support of $\vect{\widehat{\theta}}$, we have
\[
\norm*{\vect{\widehat{\theta}} - \round(\vect{\widehat{\theta}})}_2 = \norm*{\mat{G}_{\partite} \vect{\widehat{\theta}} - \mat{G}_{\partite} \round(\vect{\widehat{\theta}})}_2 = \norm*{\mat{G}_{\partite} \vect{\widehat{\theta}} - \ones}_2 \leq \sqrt{\slrrows} \delta / \gamma.
\]
Since $\vect{\widehat{\theta}} - \round(\vect{\widehat{\theta}})$ is also $k$-sparse and partite, it follows that 
\[\norm*{\mat{G}_{\sparse} \left( \vect{\widehat{\theta}} - \round(\vect{\widehat{\theta}})\right)}_2 \leq \frac{\delta \sqrt{\bdddim \slrrows} }{\gamma \sqrt{k}},\]
as each column of any block of $\mat{G}_{\sparse}$ has $\ell_2$ norm exactly $\sqrt{\bdddim/k}$. Multiplying on the left by $\mat{B}$, we have
\begin{equation}\label{eq:B-theta-hat-upper-bound}
    \norm*{\mat{B} \mat{G}_{\sparse} \vect{\widehat{\theta}} - \mat{B} \vect{z}'}_2 \leq \sigma_{\max}(\mat{B}) \cdot \norm*{\mat{G}_{\sparse} \left( \vect{\widehat{\theta}} - \round(\vect{\widehat{\theta}})\right)}_2 \leq \frac{\sigma_{\max}(\mat{B}) \cdot \delta \sqrt{\bdddim \slrrows}}{\gamma \sqrt{k}}.
\end{equation}
On the other hand, by definition of $\lambdabin(\mat{B})$, since $\vect{z} \neq \vect{z}' \in \{-1,1\}^{\bdddim}$, we know
\begin{equation}\label{eq:lambda-1-lower-bound}
    \norm{\mat{B} \vect{z} - \mat{B} \vect{z}'}_2 \geq \lambdabin(\mat{B}).
\end{equation}
By combining~\eqref{eq:B-theta-hat-upper-bound} and~\eqref{eq:lambda-1-lower-bound} by the triangle inequality, we have
\[\norm*{\mat{B} \mat{G}_{\sparse} \vect{\widehat{\theta}} - \mat{B} \vect{z} }_2 \geq \lambdabin(\mat{B}) - \frac{\sigma_{\max}(\mat{B}) \cdot \delta \sqrt{\bdddim \slrrows}}{\gamma \sqrt{k}}. \]
We now set
\begin{equation}\label{eqn:setting-gamma}
\gamma = \max\left(3 \delta \sqrt{\slrrows}, \frac{100 \cdot \sigma_{\max}(\mat{B}) \cdot \delta \sqrt{\bdddim \slrrows}}{\sqrt{k} \cdot \widehat{\lambda}_1} \right) = \Theta\left( \frac{\sigma_{\max}(\mat{B}) \cdot \sqrt{\bdddim \slrrows} }{\sqrt{k}} \right),
\end{equation}
so that 
\[\norm*{\mat{B} \mat{G}_{\sparse} \vect{\widehat{\theta}} - \mat{B} \vect{z} }_2 \geq \lambdabin(\mat{B}) - \frac{\widehat{\lambda}_1}{100} \geq \frac{49}{50} \lambdabin(\mat{B}). \]
(We remark that $\gamma$ is efficiently computable because $\sigma_{\max}(\mat{B})$ is efficiently computable and we assume $\widehat{\lambda}_1$ is known.)
By incorporating the error $\vect{e}$ and multiplication on the left by $\mat{R}$, we have
\[ \norm{\mat{X} \vect{\widehat{\theta}} - \vect{y}}_2 \geq \norm*{\mat{R} \mat{B} \mat{G}_{\sparse} \vect{\widehat{\theta}} - \mat{R} \left( \mat{B} \vect{z} + \vect{e} \right)}_2 \geq \sigma_{\min}(\mat{R}) \left(\frac{49}{50} \lambdabin(\mat{B}) - \norm{\vect{e}}_2 \right). \]
By combining this inequality with~\eqref{eq:squared-loss}, we get a contradiction as long as
\begin{equation}\label{eq:soundness}
    \delta \sqrt{\slrrows} \leq \sigma_{\min}(\mat{R}) \left(\frac{49}{50} \lambdabin(\mat{B}) - \norm{\vect{e}}_2 \right).
\end{equation} 
By Lemma~\ref{lemma:gaussian-singular-values}, we know $\sigma_{\min}(\mat{R}) \geq \sqrt{\slrrows_1}/2$ with probability at least $1 - 2e^{-\slrrows_1/32}$, and since we also know $\norm{\vect{e}}_2 \leq \alpha \lambdabin(\mat{B})$, we have 
\begin{align*}
    \sigma_{\min}(\mat{R}) \left(\frac{49}{50} \lambdabin(\mat{B}) - \norm{\vect{e}}_2 \right) &\geq \frac{\sqrt{\slrrows_1}}{2} \left(\frac{49}{50} \lambdabin(\mat{B}) - \alpha \lambdabin(\mat{B}) \right)\\
    &= \lambdabin(\mat{B}) \sqrt{\slrrows_1} \left( \frac{49}{100} - \frac{\alpha}{2} \right). 
\end{align*}
On the other hand, we know
\[ \delta \sqrt{\slrrows} = \frac{\widehat{\lambda}_1 \sqrt{3 \slrrows_1}}{10\sqrt{\slrrows}} \cdot \sqrt{\slrrows} \leq \frac{1}{5} \cdot \lambdabin(\mat{B}) \sqrt{3 \slrrows_1}. \]
Combining these inequalities, we get a contradiction if
\[ \frac{\sqrt{3}}{5} \leq \frac{49}{100} - \frac{\alpha}{2}, \]
which indeed holds if $\alpha \leq 1/10$.
\end{proof}

\section{Performance of Lasso on Our \texorpdfstring{$k$-$\SLR$}{k-SLR} Instances}\label{sec:performance-of-lasso-appendix}
In this section, we analyze the performance of the $\BinBDD$ algorithm obtained by applying Lasso to the $k$-$\SLR$ instances obtained by our reduction in Section~\ref{sec:main-reduction}. Towards this goal, we first analyze the performance of Lasso on our instances of $k$-$\SLR$.

\subsection{Normalizing Our \texorpdfstring{$k$-$\SLR$}{k-SLR} Instances}
First, we follow the convention of Negahban et al.~\cite{negahban2012unified} and normalize our design matrices according to Definition~\ref{def:column-normalization}.
\begin{lemma}\label{lemma:Z-and-spectral-norm}
Let $\mat{X} \in \mathbb{R}^{m\times n}$ denote the design matrix we obtain in Theorem~\ref{thm:main-bin-bdd-to-slr}. Let $Z = \Theta(\sigma_{\max}(\matB) \cdot \sqrt{d/k})$ be a real number. Then, with probability at least $1 - e^{-\Omega(m_1)}$, $\widetilde{\mat{X}} := \frac{1}{Z} \cdot \mat{X}$
is column-normalized, and so
\begin{equation}\label{eqn:sigma-max-X-bound}
\max_{\substack{\norm{\vect{v}}_2 = 1,\\ \norm{\vect{v}}_0 \le 2k}} \norm{\mat{X}\vect{v}}_2 \le \Theta \left(\sigma_{\max}(\mat{B}) \sqrt{dm} \right)
\end{equation}
\end{lemma}
\begin{proof}
For our instance $\mat{X}$, we have
\[ \norm*{\col_i(\mat{X})}_2^2 = \norm*{\col_i(\mat{X}_1)}_2^2 + \norm*{\col_i(\mat{X}_2)}_2^2 = \norm*{\col_i(\mat{X}_1)}_2^2 + \gamma^2. \]
Moreover,
\begin{align*}
     \norm*{\col_i(\mat{X}_1)}_2 =  \norm*{\col_i(\mat{R} \mat{B} \mat{G}_{\sparse})}_2 &= \norm*{\mat{R} \mat{B} \col_i(\mat{G}_{\sparse})}_2
    \\&\leq \sigma_{\max}(\mat{R}) \cdot  \sigma_{\max}(\mat{B}) \cdot \norm*{\col_i(\mat{G}_{\sparse})}_2
    \\&= O \left(\sqrt{\slrrows_1} \cdot  \sigma_{\max}(\mat{B}) \cdot \sqrt{\bdddim/k} \right),
\end{align*}
with probability at least $1 - e^{-\Omega(\slrrows_1)}$ over $\mat{R}$ by Lemma~\ref{lemma:gaussian-singular-values}. Since our setting of $\gamma$ (Equation~\eqref{eqn:setting-gamma}) matches this up to constant factors, we have
\[
 \norm*{\col_i(\mat{X})}_2 = O \left(\sqrt{\slrrows} \cdot  \sigma_{\max}(\mat{B}) \cdot \sqrt{\bdddim/k} \right).
\]
Therefore, we can set a normalization factor $Z = \Theta(\sigma_{\max}(\mat{B}) \cdot \sqrt{\bdddim/k})$ such that
\[ \widetilde{\mat{X}} := \frac{1}{Z} \cdot \mat{X} \]
is column-normalized. In particular, note that $\widetilde{\mat{X}}$ is scale invariant with respect to $\mat{B}$. By Lemma~\ref{lem:max-sparse-X-tilde}, we get that
\begin{align*}
\max_{\substack{\norm{\vect{v}}_2 = 1,\\ \norm{\vect{v}}_0 \le 2k}} \norm{\mat{X}\vect{v}}_2 = Z \cdot \max_{\substack{\norm{\vect{v}}_2 = 1,\\ \norm{\vect{v}}_0 \le 2k}} \norm{\widetilde{\mat{X}}\vect{v}}_2 \leq Z \cdot \sqrt{2km} = \Theta \left(\sigma_{\max}(\mat{B}) \sqrt{dm} \right)
\end{align*}
\end{proof}

\subsection{Restricted Eigenvalue Condition}\label{sec:proof-of-re-constant-appendix}

Now, we prove the lower bound on the restricted eigenvalues of the (column-normalized) design matrices of our $k$-$\SLR$ instances.

\begin{lemma}\label{lemma:re-constant-lower-bound}
Suppose $d \geq 3k$. For a setting of $\REfac = \Theta(k/d)$, $\widetilde{\mat{X}}$ satisfies the $(\RE, \REfac)$-restricted eigenvalue condition for all $k$-sparse, $k$-partite $S \subseteq [\slrcols]$ for some
\begin{align*}
    \RE &= \Theta\left( \frac{k}{\bdddim^4 \cdot \cond(\mat{B})^2} \right),
\end{align*}
with probability at least $1 - e^{-\Omega(m)}$.
\end{lemma}

\begin{proof}[Proof of Lemma~\ref{lemma:re-constant-lower-bound}]
For simplicity, we will show that the unnormalized design matrix $\mat{X} = Z \cdot  \widetilde{\mat{X}}$ has the property that for all $k$-sparse, $k$-partite $S \subset [n]$, and for all $\vect{\theta} \in \mathbb{C}_\epsilon(S)$,
\begin{align*}
    \frac{\|\mat{X} \vect{\theta}\|_2^2}{m \cdot \|\vect{\theta}\|_2^2} \ge \RE'.
\end{align*}
We will then set $\RE = \RE'/Z^2$ to get the final bound.

Suppose for the sake of contradiction that there exists $\vect{\theta} \in \cone_{\REfac}(S)$ such that
\[ \frac{\norm{\mat{X} \vect{\theta}}_2^2}{m \cdot \norm{\vect{\theta}}_2^2} < \RE'. \]
Without loss of generality, we can scale $\vect{\theta}$ so that $\norm{\vect{\theta}}_1 = 1$. In particular, observe that $\norm{\vect{\theta}}_2^2 \leq \norm{\vect{\theta}}_1^2 = 1$. We have
\begin{align*}
\RE' > \frac{\|\mat{X} \vect{\theta}\|_2^2 }{ m \cdot \| \vect{\theta}\|_2^2} &= \frac{\|\mat{R} \mat{B} \mat{G}_\sparse \vect{\theta}\|_2^2 + \|\mat{G}_\partite \vect{\theta}\|_2^2}{m \cdot \|\vect{\theta}\|_2^2}\\
&\ge \frac{1}{m} \left(\sigma_{\min}(\mat{R} \mat{B})^2 \|\mat{G}_\sparse \vect{\theta}\|_2^2 + \|\mat{G}_\partite \vect{\theta}\|_2^2 \right).
\end{align*}
If $\|\mat{G}_\partite\vect{\theta}\|_2^2 \ge \RE' m$, we have arrived at a contradiction. Otherwise, we will show that $\|\mat{G}_\sparse \vect{\theta}\|_2^2 \ge \RE' m / \sigma_{\min}(\mat{R} \mat{B})^2$, which would also be a contradiction. Indeed, since $\mat{G}_{\sparse} \vect{\theta} \in \R^{d}$, we know $\norm{\mat{G}_{\sparse} \vect{\theta}}_2 \sqrt{d} \geq \norm{\mat{G}_{\sparse} \vect{\theta}}_1$, so it suffices to show that $\|\mat{G}_\sparse \vect{\theta}\|_1 \ge \tau_1$, where we define
\begin{align*}
    \tau_1 := \frac{\sqrt{\RE' d m} }
    {\sigma_{\min}(\mat{R} \mat{B})}.
\end{align*}
Since $\vect{\theta} \in \mathbb{C}_\epsilon(S)$ for some $k$-partite $k$-sparse $S \subseteq [n]$, we can write $\vect{\theta} = \vect{\theta}_S + \vect{\theta}_{\bar{S}}$ such that $\|\vect{\theta}_{\bar{S}}\|_1 \le (1 + \epsilon) \|\vect{\theta}_S\|_1$. Let $\eta^{(i)} \in \R$ be the $i$th nonzero entry of $\vect{\theta}_S$, and let $\vect{\bar{\eta}}^{(i)} \in \R^{n/k}$ be the $i$th part of $\vect{\theta}_{\bar{S}}$. Note that without loss of generality, we can assume that all the entries of $\vect{\theta}_S$ are non-negative, that is, for all $i$, $\eta^{(i)} \ge 0$. If not, suppose $\eta^{(i)} < 0$ for some $i \in [k]$. Because of the block diagonal structure of $\mat{G}_\sparse$, we can simply negate the $i$th part of $\vect{\theta}$ so that none of the norms change.

By applying Lemma~\ref{lemma:averaging-argument} to the $\ell_1$ norm of the parts of $\vect{\theta}_S$ and $\vect{\theta}_{\bar{S}}$, we know that there is some part $i \in [k]$ such that $\eta^{(i)} \ge \epsilon/k$ and $\|\vect{\bar{\eta}}^{(i)}\|_1 \le (1 + 10 \epsilon) \eta^{(i)}$. We will consider this part $i$, and we will write $\eta, \vect{\bar{\eta}}$ instead of $\eta^{(i)}, \vect{\bar{\eta}}^{(i)}$ for simplicity, i.e.,
\begin{align}\label{eqn:avg-rcs}
    \sum_{j} |\bar{\eta}_j| = \|\vect{\bar{\eta}}\|_1 \le (1 + 10\epsilon) \eta.
\end{align}
We will provide a lower bound on $\|\mat{G}_\sparse \vect{\theta}\|_1$ considering only the $i$th part. Since $\|\mat{G}_\partite \vect{\theta}\|_2 \le \sqrt{\RE' m}$,
\begin{align}\label{eqn:g-partite-kernel}
     \eta + \sum_{j} \bar{\eta}_j \le \left| \eta + \sum_{j} \bar{\eta}_j \right| &\le \frac{1}{\gamma}\|\mat{G}_\partite \vect{\theta}\|_2 \le \frac{\sqrt{\RE' m}}{\gamma}.
\end{align}
Adding~\eqref{eqn:avg-rcs} and \eqref{eqn:g-partite-kernel}, we get an upper bound on the sum of the positive entries in $\vect{\bar{\eta}}$:
\begin{align}\label{eqn:small-positive-cancellation}
    \sum_{j} \bar{\eta}_j \cdot \mathbbm{1}[\bar{\eta}_j > 0] \le \frac{\sqrt{\RE' m}}{2 \gamma} + 5 \epsilon \eta.
\end{align}
Let $\mat{H} \in \{-1,1\}^{d/k \times 2^{d/k}}$ be a block in $\mat{G}_{\sparse}$, i.e., whose columns are all of the vectors in $\{-1,1\}^{d/k}$. Let $\vect{h}^{(j)}$ denote the $j$th column of $\mat{H}$. Without loss of generality, we can choose the location of the $\eta$ entry within the block so that it corresponds to the column $\ones \in \R^{d/k}$. To see this, if not, we can flip the sign of the rows of $\mat{H}$ that correspond to $-1$ entries in the column corresponding to $\eta$. This has the effect of permuting the columns of $\mat{H}$ and has no effect on $\norm{\mat{G}_{\sparse} \vect{\theta}}_1$ or $\norm{ \mat{G}_{\partite} \vect{\theta}}_2$ as the entries just flip sign. We will use the convention that first column of $\mat{H}$ is $\ones$. By considering only part $i$, we have the inequality
\begin{equation}\label{eqn:decomposing-sparse-gadget-contributions}
 \norm{\mat{G}_{\sparse} \vect{\theta}}_1 \geq \norm*{\eta \ones + \sum_{j = 2}^{2^{d/k}} \bar{\eta}_j \vect{h}^{(j)}}_1 \geq \sum_{\ell = 1}^{d/k} \left( \eta + \sum_{j = 2}^{2^{d/k}} \bar{\eta}_j h^{(j)}_{\ell} \right) = \frac{\eta d}{k} + \sum_{j = 2}^{2^{d/k}} \bar{\eta}_j \sum_{\ell = 1}^{d/k}  h^{(j)}_{\ell}.
 \end{equation}
Since $\vect{h}^{(j)} \neq \ones$ for all $j \geq 2$, we know for all $j \geq 2$,
\[ \sum_{\ell = 1}^{d/k}  h^{(j)}_{\ell} \leq \frac{d}{k}-2, \]
as at least one entry in $\vect{h}^{(j)}$ is $-1$. This bound is helpful only when $\bar{\eta}_j < 0$, as this will let us lower bound $\norm{\mat{G}_{\sparse} \vect{\theta}}_1$. Thankfully, \eqref{eqn:small-positive-cancellation} gives us a bound on the sum of the positive $\bar{\eta}_j$ contributions, in which case we use the trivial bound
\[ \sum_{\ell = 1}^{d/k}  h^{(j)}_{\ell} \geq -\frac{d}{k}. \]
Let $J_{+}, J_{-} \subseteq [2^{d/k}] \setminus \{1\}$ denote the set of columns $j$ where $\bar{\eta}_j \geq 0$ and $\bar{\eta}_j < 0$, respectively. Combining the above bounds with~\eqref{eqn:decomposing-sparse-gadget-contributions}, we have
\begin{align}
\norm{\mat{G}_{\sparse} \vect{\theta}}_1 \geq \frac{\eta d}{k} + \sum_{j = 2}^{2^{d/k}} \bar{\eta}_j \sum_{\ell = 1}^{d/k} &= \frac{\eta d}{k} + \sum_{j \in J_{+}} \bar{\eta}_j \sum_{\ell = 1}^{d/k} h^{(j)}_{\ell} + \sum_{j \in J_{-}} \bar{\eta}_j \sum_{\ell = 1}^{d/k} h^{(j)}_{\ell}
\\&\geq \frac{\eta d}{k} - \frac{d}{k} \sum_{j \in J_{+}} \bar{\eta}_j - \left(\frac{d}{k} - 2 \right) \left| \sum_{j \in J_{-}} \bar{\eta}_j \right|
\\&\geq \frac{\eta d}{k} - \frac{d}{k} \left( \frac{\sqrt{\RE' m}}{2 \gamma} + 5 \epsilon \eta \right) - \left(\frac{d}{k} - 2 \right) \sum_{j = 2}^{2^{d/k}} \left| \bar{\eta}_j \right|
\\&\geq \frac{\eta d}{k} - \frac{d}{k} \left( \frac{\sqrt{\RE' m}}{2 \gamma} + 5 \epsilon \eta \right) - \left(\frac{d}{k} - 2 \right) (1 + 10\epsilon) \eta\label{eqn:final-Gsparse-bound},
\end{align}
where the penultimate inequality comes from~\eqref{eqn:small-positive-cancellation}, and the last inequality comes from~\eqref{eqn:avg-rcs}. We now bound the remaining terms. Setting
\[ \epsilon := \frac{k}{50d}, \]
and assuming 
\begin{equation}\label{eqn:gamma-assumption} \gamma \geq \frac{5d \sqrt{\RE' m}}{ \eta k},
\end{equation}
we have the guarantee that
\[\frac{\sqrt{\RE' m}}{2 \gamma} + 5 \epsilon \eta \leq \frac{\eta k}{10d} +  \frac{\eta k}{10d} = \frac{\eta k}{5d}. \]
Furthermore, our setting of $\epsilon$ guarantees 
\[ \left(\frac{d}{k} - 2 \right) (1 + 10\epsilon) =  \frac{d}{k} - \frac{9}{5} - \frac{2k}{5d} < \frac{d}{k} - 1.\]
Plugging these into~\eqref{eqn:final-Gsparse-bound}, we have
\begin{align*}
    \norm{\mat{G}_{\sparse} \vect{\theta}}_1 &\geq \frac{\eta d}{k} - \frac{d}{k} \left( \frac{\sqrt{\RE' m}}{2 \gamma} + 5 \epsilon \eta \right) - \left(\frac{d}{k} - 2 \right) (1 + 10\epsilon) \eta
    \\&\geq \frac{\eta d}{k} - \frac{d}{k} \cdot \frac{\eta k}{5d} - \left(\frac{d}{k} - 1 \right) \eta
    \\&= \frac{4}{5} \eta \geq \frac{4 \epsilon}{5k} = \frac{2}{125d},
\end{align*}
where the last inequality comes from our application of Lemma~\ref{lemma:averaging-argument}, which implies $\eta \geq \epsilon/k$. Recall that we get a contradiction if
\[ \|\mat{G}_\sparse \vect{\theta}\|_1 \ge \tau_1 := \frac{\sqrt{\RE' d m} }
    {\sigma_{\min}(\mat{R} \mat{B})}.\]
Therefore, setting $\RE'$ so that
\[ \frac{\sqrt{\RE' d m} }
    {\sigma_{\min}(\mat{R} \mat{B})} < \frac{2}{125d}\]
arrives us at a contradiction. Since $\sigma_{\min}(\mat{R} \mat{B}) \geq \sigma_{\min}(\mat{R}) \sigma_{\min}(\mat{B}) \geq \sqrt{m_1}/2 \cdot \sigma_{\min}(\mat{B})$ with probability at least $1 - e^{-\Omega(m_1)}$ by Lemma~\ref{lemma:gaussian-singular-values}, it suffices to set
\[ \RE' := \frac{1}{2} \cdot \frac{m_1 \sigma_{\min}(\mat{B})^2}{125^2 d^3 m} = \Theta \left( \frac{\sigma_{\min}(\mat{B})^2}{d^3} \right). \]

Now, we justify our assumption on $\gamma$ in~\eqref{eqn:gamma-assumption}. Recall that our reduction sets
\[ \gamma = \Theta \left( \frac{ \sigma_{\max}(\mat{B}) \sqrt{dm}}{\sqrt{k}} \right). \]
By our setting of $\RE'$ and since $\eta \geq \epsilon / k = 1/(50d)$, we have
\[ \gamma \gg \frac{\sigma_{\min}(\mat{B})\sqrt{dm}}{k} \geq \Theta\left(\frac{\sigma_{\min}(\mat{B})\sqrt{m}}{k \eta \sqrt{d}} \right) = \Theta \left( \frac{d \sqrt{\RE' m}}{\eta k} \right), \]
showing that~\eqref{eqn:gamma-assumption} is indeed satisfied.

Finally, scaling by $Z^2$, we have
\[ \RE = \frac{\RE'}{Z^2} = \Theta \left( \frac{k \cdot \sigma_{\min}(\mat{B})^2}{d^4 \cdot \sigma_{\max}(\mat{B})^2} \right),\]
as desired.
\end{proof}

\begin{lemma}\label{lemma:averaging-argument}
    Suppose $x_i, y_i \in \R_{\geq 0}$ for $i \in [k]$, such that $\sum_{i=1}^k x_i + \sum_{i=1}^k y_i = 1$. Further suppose that for some $\epsilon \in (0, 1/20)$, $\sum_{i=1}^k y_i \le (1 + \epsilon) \sum_{i=1}^k x_i$. Then, there is some $i^* \in [k]$ such that $x_{i^*} \ge \epsilon/k$ and $y_{i^*} \le (1 + 10 \epsilon) x_{i^*}$.
\end{lemma}

\begin{proof}
Since $\sum_{i=1}^k x_i + \sum_{i=1}^k y_i = 1$ and $\sum_{i=1}^k y_i \le (1 + \epsilon) \sum_{i=1}^k x_i$, this implies that $\sum_{i=1}^k x_i \ge \frac{1}{2 + \epsilon} \ge \frac{1}{2} - \epsilon$, and $\sum_{i=1}^k y_i \le \frac{1}{2} + \epsilon$.

Suppose for contradiction that for all $i \in [k]$, either $x_i < \epsilon / k$ or $y_i > (1 + 10 \epsilon) x_i$. Let $A$ be the set of all $i \in [k]$ such that $x_i < \epsilon / k$, and let $B = [k] \setminus A$. By assumption, for all $i \in B$, we have that $y_i > (1 + 10 \epsilon) x_i$. Then, since $\sum_{i \in A} x_i \le \epsilon$,
\begin{align*}
    \sum_{i \in B} y_i &> (1 + 10\epsilon) \sum_{i \in B} x_i = (1 + 10\epsilon) \left( \sum_{i \in [k]} x_i - \sum_{i \in A} x_i \right)\\
    &\ge (1 + 10 \epsilon) \left(\frac{1}{2} - 2 \epsilon \right) = \frac{1}{2} + 5 \epsilon - 2 \epsilon - 20 \epsilon^2 \ge \frac{1}{2} + 2 \epsilon.
\end{align*}
This is a contradiction, since we also know that $\sum_{i \in B} y_i \le \sum_{i \in [k]} y_i \le \frac{1}{2} + \epsilon$.
\end{proof}

\subsection{Prediction Error Achieved by Lasso on Our Instances}\label{sec:lasso-on-our-instances-appendix}
We now apply Theorem~\ref{thm:lasso-performance-in-terms-of-noise} to our instance as generated in Theorem~\ref{thm:main-bin-bdd-to-slr}, to obtain a bound on the prediction error achieved by Lasso on our instances. Note that this is a mean squared prediction error bound as opposed to a mean squared error bound. To translate Lasso into an algorithm for $\BinBDD_{d, \alpha}$ via the reduction in Theorem~\ref{thm:main-bin-bdd-to-slr}, we can bound the mean-squared error by the triangle inequality (see Corollary~\ref{cor:lasso-binbdd}).
\begin{lemma}\label{lemma:prediction-error-lasso-on-our-instances}
Consider the instances $(\matX, \vect{y}, \delta = \Theta(\lambda_{1, \bin}(\matB))$ of $k$-$\SLR$ instances obtained in the reduction from a $\BinBDD_{d, \alpha}$ instance $(\matB, \cdot)$ as in Theorem~\ref{thm:main-bin-bdd-to-slr}. On these instances, the thresholded Lasso estimator $\vect{\theta}_{\mathsf{TL}}$ can achieve prediction error bounded as
\begin{align*}
    \frac{1}{m}\norm{ \mat{X}\vect{\widehat{\theta}}_{\TL} - \mat{X} \vect{\theta}^*}^2_2 
    \leq O \left( \frac{\alpha^2 \cdot \lambdabin(\mat{B})^2 \cdot \cond(\mat{B})^4 \cdot d^{10}}{k^2}\right) .
\end{align*}
\end{lemma}

\begin{proof}
    With the goal of applying Theorem~\ref{thm:lasso-performance-in-terms-of-noise}, we first bound $ \norm*{\widetilde{\vect{w}}}_2$. On our instances, we have
    \[ \widetilde{\vect{w}} = \widetilde{\vect{y}} - \widetilde{\mat{X}} \vect{\theta}^* =  \frac{1}{Z} \begin{pmatrix}
\mat{R} \vect{e}\\
\zero
\end{pmatrix}.\]
Therefore,
    \[ \norm*{\widetilde{\vect{w}}}_2 = \frac{1}{Z} \norm{\mat{R} \vect{e}}_2 \leq \frac{1}{Z} \sigma_{\max}(\mat{R}) \norm{\vect{e}}_2 \leq O \left( \frac{ \alpha \cdot  \lambdabin(\mat{B}) \sqrt{\slrrows}}{Z} \right)\]
with probability at least $1 - e^{-\Omega(m)}$ over $\mat{R}$ by Lemma~\ref{lemma:gaussian-singular-values}. Since $Z = \Theta(\sigma_{\max}(\mat{B}) \cdot \sqrt{\bdddim/k})$, we have
    \[ \norm*{\widetilde{\vect{w}}}_2 \leq  O \left( \frac{ \alpha \cdot  \lambdabin(\mat{B}) \sqrt{mk}}{\sigma_{\max}(\matB) \sqrt{d}} \right).\]
With this bound on $ \norm*{\widetilde{\vect{w}}}_2$, we can plug in Lemma~\ref{lemma:re-constant-lower-bound}, which says that for $\epsilon = \Theta(k/d)$, $\widetilde{\mat{X}}$ satisfies the $(\RE, \epsilon)$-restricted eigenvalue condition for all $k$-sparse, $k$-partite $S \subseteq [n]$ for 
\[ \RE = \Theta\left( \frac{k}{\bdddim^4 \cdot \cond(\mat{B})^2} \right).\]
Plugging in $ \norm*{\widetilde{\vect{w}}}_2$, $\epsilon$, and $\RE$ into Theorem~\ref{thm:lasso-performance-in-terms-of-noise}, we get
    \[\frac{1}{m} \| \widetilde{\mat{X}} \widehat{\vect{\theta}}_{\mathsf{TL}} - \widetilde{\mat{X}} \vect{\theta}^*\|_2^2 \le O\left( \frac{\norm{\widetilde{\vect{w}}}_2^2 \cdot k^2}{\RE^2 \cdot  \epsilon^2 \cdot m} \right) \leq O \left( \frac{\alpha^2 \cdot \lambdabin(\matB)^2 \cdot \cond(\matB)^4 \cdot d^9}{k \cdot \sigma_{\max}(\matB)^2} \right). \]
Scaling up by $Z^2 = \Theta(\sigma_{\max}(\matB)^2 \cdot d/k)$ to convert back to the unnormalized version, we have
    \[\frac{1}{m} \| \mat{X} \widehat{\vect{\theta}}_{\mathsf{TL}} - \mat{X} \vect{\theta}^*\|_2^2 \leq O \left( \frac{\alpha^2 \cdot \lambdabin(\matB)^2 \cdot \cond(\matB)^4 \cdot d^{10}}{k^2} \right), \]
as desired.

\end{proof}

\begin{corollary}[Lasso-based Algorithm for $\BinBDD_{d, \alpha}$]\label{cor:lasso-binbdd}
Let $d, k, m$ be integers such that $k$ divides $d$ and $m\ge \Omega(d)$. Let $\alpha$ be a real number such that
\begin{align*}
    \alpha \leq C \cdot \frac{k}{d^5 \cdot \cond(\mat{B})^2}
\end{align*}
for some universal constant $C>0$. Then, running the reduction in Theorem~\ref{thm:main-bin-bdd-to-slr} and reducing to Lasso as the $k$-$\SLR$ solver gives a $\poly(m, k \cdot 2^{d/k})$-time reduction from $\BinBDD_{d, \alpha}$ that succeeds with probability at least $1 - e^{-\Omega(m)}$. Setting $k$ and $m$ such that $d = 10k$ and $m = 100d$, this becomes a $\poly(d)$ time reduction from $\BinBDD_{d, \alpha}$ to Lasso that succeeds with probability at least $1 - e^{-\Omega(d)}$ as long as
\begin{align*}
    \alpha \le C' \cdot \frac{1}{d^4 \cdot \cond(\mat{B})^2},
\end{align*}
for some universal constant $C'>0$.
\end{corollary}

\begin{proof}
Let $(\mat{X}, \vect{y} = \mat{X} \vect{\theta}^* + \vect{w}, \delta)$ be the $k$-$\SLR$ instance obtained in the reduction in Theorem~\ref{thm:main-bin-bdd-to-slr}. Then, if $\widehat{\vect{\theta}}_{\mathsf{TL}}$ is the solution obtained via thresholded Lasso, since $\|\mat{X}\vect{\theta}^* - \vect{y}\|^2_2/m = \|\mat{R} \vect{e}\|^2_2/m \le O(\alpha^2 \cdot \lambda_{1,\bin}(\mat{B}))$ with probability at least $1 - e^{-\Omega(m)}$, by triangle inequality, we have that
\begin{align*}
    \frac{1}{m} \|\mat{X} \widehat{\vect{\theta}}_{\mathsf{TL}} - \vect{y}\|^2_2 
    &\le O \left(\frac{1}{m} \|\mat{X} 
    \widehat{\vect{\theta}}_{\mathsf{TL}} - \mat{X} \vect{\theta}^* \|^2_2 + \frac{1}{m} 
    \|\mat{X}\vect{\theta}^* - \vect{y}\|^2_2\right)\\
    &\le O\left(\alpha^2 \cdot \lambda_{1, \bin}(\mat{B})^2 \cdot \cond(\mat{B})^4 \cdot \frac{d^{10}}{k^2}\right) + O(\alpha^2 \cdot \lambda_{1, \bin}(\mat{B})^2).
\end{align*}
Theorem~\ref{thm:main-bin-bdd-to-slr} needs a $k$-$\SLR$ solver with $\delta = \Theta(\lambda_{1, \bin}(\mat{B}))$ by~\eqref{eqn:setting-delta}, so setting
\begin{align*}
    \alpha \leq C \cdot \frac{k}{d^5 \cdot \cond(\mat{B})^2}
\end{align*}
for some universal constant $C > 0$ suffices for the correctness of the reduction. This gives us a Lasso-based algorithm for $\BinBDD_{d, \alpha}$ that runs in time $\poly(m, k \cdot 2^{\bdddim/k})$.
\end{proof}

\begin{remark}
We briefly remark that the information-theoretic $k$-$\SLR$ solver (i.e., by the bound in Proposition~\ref{prop:worst-case-info-theoretic-bound}) would solve our BDD instance for sufficiently small \emph{constant} $\alpha$. In particular, if allowed oracle calls to an information-theoretic $k$-$\SLR$ solver, there would be a $\poly(d)$-time algorithm for $\BinBDD_{d, \alpha}$ for sufficiently small constant $\alpha$.
\end{remark}

Finally, we state how a hypothetical improvement to Lasso gives a corresponding parameter improvement to $\BinBDD_{d, \alpha}$.

\begin{theorem}\label{thm:better-lasso-gives-better-bdd}
Let $d, k, m, n$ be integers such that $k$ divides $d$, $m \ge \Omega(d)$ and $n = k \cdot 2^{d/k}$.
Suppose there is a polynomial-time algorithm that takes $k$-$\SLR$ instances $(\widetilde{\mat{X}}, \widetilde{\vect{y}} = \widetilde{\mat{X}} \vect{\theta}^* + \widetilde{\vect{w}})$ with design matrices $\widetilde{\mat{X}} \in \R^{m \times n}$ that satisfy the $(\RE, \REfac)$-restricted eigenvalue condition for $\support(\vect{\theta}^*)$, and outputs a $k$-sparse vector $\widehat{\vect{\theta}}$ such that the prediction error is bounded as
\begin{align*}
    \frac{1}{m} \|\widetilde{\mat{X}} \widehat{\vect{\theta}} - \widetilde{\mat{X}} \vect{\theta}^* \|_2^2 \le O \left( \frac{\|\widetilde{\vect{w}}\|_2^2 \cdot k^2}{\RE^{2 - \beta} \cdot \epsilon^2 \cdot m }\right)
\end{align*}
for some constant $\beta \in (0, 2]$. Then there is an algorithm that runs in time $\poly(m, k \cdot 2^{d/k})$ that solves $\BinBDD_{d, \alpha}$ with probability $1 - e^{- \Omega(m)}$ on instances $(\mat{B}, \cdot)$ as long as
\begin{align*}
    \alpha \le C \cdot \frac{k^{1 - \beta/2}}{d^{5 - 2 \beta} \cdot \cond(\mat{B})^{2 - \beta}},
\end{align*}
for some universal constant $C > 0$. 
\end{theorem}

\begin{proof}[Proof of Theorem~\ref{thm:better-lasso-gives-better-bdd}]
    The proof follows directly by modifying the exponent on $\RE$ in Lemma~\ref{lemma:prediction-error-lasso-on-our-instances} and Corollary~\ref{cor:lasso-binbdd}.
\end{proof}

\section{Reduction from CLWE}\label{sec:reduction-from-clwe}
We now define the continuous learning with errors ($\CLWE$) assumption as defined by~\cite{BRST20}, which holds assuming the quantum~\cite{BRST20} or classical~\cite{gupte2022continuous} hardness of worst-case lattice problems for various parameter regimes. We define $\sphere^{m-1}$ to be the unit sphere in $\R^m$, i.e., the set of all unit vectors in $\R^m$ according to the $\ell_2$ norm. We write $\vect{v} \sim \sphere^{m-1}$ to denote sampling $\vect{v}$ from the uniform distribution over the sphere. Furthermore, for a vector $\vect{v} \in \R^n$, by $\vect{v} \pmod{1}$, we mean the representative $\widetilde{\vect{v}} \in [-1/2, 1/2)^n$ such that $\vect{v} - \widetilde{\vect{v}} \in \Z^n$. Moreover, by this choice of representative, note that we have $| a \pmod{1}| \leq |a|$ for all $a \in \R$.

\begin{definition}[CLWE Assumption~\cite{BRST20}]\label{def:clwe} For a secret vector $\vect{s} \sim \gamclwe \cdot \sphere^{m-1}$ and error vector $\vect{e} \sim \N(0, \beta^2 I_{n \times n})$, we have the computational indistinguishability
\begin{align*} &\left( \mat{A} \sim \N(0, 1)^{m \times n}, \vect{b}^\top := \vect{s}^\top \mat{A} + \vect{e}^\top \pmod{1} \right)
\\\approx&\left( \mat{A} \sim \N(0, 1)^{m \times n}, \vect{b}^\top \sim [-1/2, 1/2)^{m} \right),
\end{align*}
up to constant advantage for algorithms running in time $T$.
\end{definition}
This assumption is parameterized by $m, n, \gamclwe, \beta$ and $T$. For example, for the parameter setting $\gamclwe = 2 \sqrt{m}$ and $\beta = 1/\poly(m)$, there are no algorithms known that run in time $T = 2^{o(m)}$ with $n = 2^{o(m)}$ samples.

\begin{theorem}\label{thm:reduction-from-clwe}
Suppose $k = m^{\epsilon}$ for some $\epsilon \in (0,1)$. There is a $\poly(n, m)$-time reduction from $\CLWE$ in dimension $m$ with $n$ samples and parameters $\gamclwe = 2 \sqrt{m}, \beta = o(1/\sqrt{k})$ to $k$-$\SLR$, for $n = k \cdot 2^{\Theta( m/k \cdot  \log m)}$. If the $k$-$\SLR$ solver runs in time $T$, the distinguisher for $\CLWE$ runs in time $T + \poly(n,m)$. Moreover, the reduction generates design matrices $\mat{X}$ that can be decomposed as
\[ \mat{X} = \begin{pmatrix}
\mat{X}_1\\
\mat{X}_2
\end{pmatrix}, \]
where the distribution of each row of $\mat{X}_1 \in \R^{\slrrows \times \slrcols}$ is $\iid$ $\N(\zero, I_{n \times n})$, and $\mat{X}_2 \in \R^{k \times \slrcols}$ is proportional to a fixed, instance-independent matrix that depends only on $\slrcols$ and $k$. 
\end{theorem}

The interpretations of Section~\ref{sec:removing-worst-case-gadget} also apply here as to how one could remove the worst-case sub-matrix $\mat{X}_2$.

\begin{proof}
We first start with the $\CLWE$ assumption (see Definition~\ref{def:clwe}), which says that for secret vector $\vect{s} \sim \gamclwe \cdot \sphere^{m-1}$ and error vector $\vect{e} \sim \N(0, \beta^2 I_{n \times n})$, we have the indistinguishability
\begin{align*} &\left( \mat{A} \sim \N(0, 1)^{m \times n}, \vect{b}^\top := \vect{s}^\top \mat{A} + \vect{e}^\top \pmod{1} \right)
\\\approx&\left( \mat{A} \sim \N(0, 1)^{m \times n}, \vect{b}^\top \sim [-1/2, 1/2)^{m} \right).
\end{align*}
We construct our design matrix $\mat{X} \in \R^{(m+k) \times n}$ and target $\vect{y} \in \R^{m+k}$ as
\begin{align*}
    \mat{X} :=
    \begin{pmatrix}
    \mat{A} \\
    \alpha \mat{G}_{\partite}
    \end{pmatrix}, \quad
    \vect{y} = \begin{pmatrix}
    \mathbf{0}\\
    \alpha \ones
    \end{pmatrix},
\end{align*}
with parameter 
\[ \delta := \frac{1}{100 \gamclwe \sqrt{m+k}} = \Theta \left( \frac{1}{\gamclwe \sqrt{m + k}} \right). \]
For a ($k$-sparse) $k$-$\SLR$ solution $\vect{\widehat{\theta}} \in \R^{n}$ to the instance $(\mat{X}, \vect{y}, \delta)$, we compute $\round(\vect{\widehat{\theta}}) \in \Z^n$ and check if
\[ \left| \vect{b}^\top \round(\vect{\widehat{\theta}}) \pmod{1} \right| < \frac{1}{4}. \]
If so, we output $1$ (indicating $\CLWE$), and otherwise, we output $0$ (indicating null).

\paragraph{Soundness.} To see that we have a distinguisher for $\CLWE$, suppose $\vect{\widehat{\theta}} \in \R^n$ is a $k$-$\SLR$ solution, i.e.,
\begin{equation}\label{eqn:squared-loss-clwe}
    \frac{\norm{\mat{X} \vect{\widehat{\theta}} - \vect{y}}_2^2}{m+k} \leq \delta^2. 
    \end{equation} 
In particular, we know $\norm{\alpha \mat{G}_{\partite} \vect{\widehat{\theta}} - \alpha \ones}_2^2 \leq (\slrrows + k) \delta^2 = \frac{1}{100^2 \gamclwe^2}$, or equivalently,
\[
\norm{\mat{G}_{\partite} \vect{\widehat{\theta}} - \ones}_2 \leq \frac{1}{100 \alpha \gamclwe}.
\]
Throughout, assume we set $\alpha$ so that $ 100 \alpha \gamclwe > 2$. Each part of $\vect{\widehat{\theta}}$ must have exactly one non-zero entry, as $\vect{\widehat{\theta}}$ is $k$-sparse and cannot have any part with all zero entries. Moreover, we know this non-zero entry must be in $[1 - 1/(100 \alpha \gamclwe), 1 + 1/(100 \alpha \gamclwe)]$, so this entry must round to $1$. Therefore, $\round(\vect{\widehat{\theta}}) \in \S_{\slrcols, k}$, and $\round(\vect{\widehat{\theta}})$ and $\vect{\widehat{\theta}}$ have the same support. By considering the support of $\vect{\widehat{\theta}}$, this implies
\[ \norm{\round(\vect{\widehat{\theta}}) - \vect{\widehat{\theta}}}_2 = \norm{\mat{G}_{\partite} \vect{\widehat{\theta}} - \ones}_2 \leq \frac{1}{100 \alpha \gamclwe}. \] 
Moreover, we will set $\alpha = \max(\sqrt{n}, 3/(100 \gamclwe))$ so that
\[ \norm{\round(\vect{\widehat{\theta}}) - \vect{\widehat{\theta}}}_2  < \frac{1}{100 \gamclwe \sqrt{n}}.\] By~\eqref{eqn:squared-loss-clwe}, we also know
\[ \norm{\mat{A} \vect{\widehat{\theta}}}_2 \leq \delta \sqrt{m+k} \leq \frac{1}{100 \gamclwe}. \]
First, suppose we are the non-null case, i.e., $\vect{b}^\top = \vect{s}^\top \mat{A} + \vect{e}^{\top}$. Since $\round(\vect{\widehat{\theta}}) \in \Z^n$, we then have
\begin{align*} \left| \vect{b}^\top \round(\vect{\widehat{\theta}}) \pmod{1} \right| &= \left|\vect{s}^\top \mat{A} \round(\vect{\widehat{\theta}}) + \vect{e}^\top \round(\vect{\widehat{\theta}}) \pmod{1} \right|
\\&\leq \left| \vect{s}^\top \mat{A} (\round(\vect{\widehat{\theta}}) - \vect{\widehat{\theta}}) \right| + \left| \vect{s}^\top \mat{A} \vect{\widehat{\theta}} \right| + \left| \vect{e}^\top \round(\vect{\widehat{\theta}}) \right|
\\&\leq \norm{\vect{s}^\top \mat{A}}_2 \cdot \norm{\round(\vect{\widehat{\theta}}) - \vect{\widehat{\theta}}}_2 + \frac{1}{100} + \left| \vect{e}^\top \round(\vect{\widehat{\theta}}) \right|. 
\end{align*}
Since $\mat{X}$ and $\vect{y}$ are independent of $\vect{e}$, and since $\round(\vect{\widehat{\theta}}) \in \S_{n,k}$, we know $\vect{e}^\top \round(\vect{\widehat{\theta}}) \sim \N(0, k \beta^2)$. By standard tails bounds on the Gaussian, we have $ | \vect{e}^\top \round(\vect{\widehat{\theta}})| \leq O(\beta \sqrt{k})$ with probability at least $99/100$. Similarly, since $\vect{s}$ and $\mat{A}$ are independent, we know $\vect{s}^\top \mat{A} \sim \N(0, \gamclwe I_{n \times n})$, so by Lemma~\ref{lemma:chi-squared-tail-bound}, we know $\norm{\vect{s}^\top \mat{A}}_2 \leq 10 \gamclwe \sqrt{n}$ with probability at least $1 - e^{-\Omega(n)}$. Plugging these into our previous bound, for $\beta = o(1/\sqrt{k})$, 
\begin{align*} \left| \vect{b}^\top \round(\vect{\widehat{\theta}}) \pmod{1} \right| &\leq \norm{\vect{s}^\top \mat{A}}_2 \cdot \norm{\round(\vect{\widehat{\theta}}) - \vect{\widehat{\theta}}}_2 + \frac{1}{100} + \left| \vect{e}^\top \round(\vect{\widehat{\theta}}) \right|
\\&\leq \frac{1}{10} + \frac{1}{100} + o(1) < \frac{1}{4},
\end{align*}
with probability at least $\frac{49}{50}$.

On the other hand, if we are in the null case, then $\vect{b} \sim [-1/2, 1/2)^m$ independently of $\vect{\widehat{\theta}}$, and consequently $\vect{b}^\top \round(\vect{\widehat{\theta}}) \pmod{1} \sim [-1/2, 1/2)$, in which case
$ \Pr\left[ \left| \vect{b}^\top \round(\vect{\widehat{\theta}}) \pmod{1} \right| \geq 1/4 \right] = 1/2$.

Therefore, we have a distinguisher for $\CLWE$ with $\Omega(1)$ advantage.

\paragraph{Completeness.}
As a warm-up, we compute the expected number of solutions $\vect{\theta} \in \S_{n,k}$. In particular, since $\vect{\theta} \in \S_{n,k}$, we have $\mat{G}_{\partite} \vect{\theta} = \ones$ by Lemma~\ref{lemma:G-forces-partite-binary}. It therefore suffices to show
\[ \frac{\norm{\mat{A} \vect{\theta}}_2^2}{m+k} \leq \delta^2 \]
for $\vect{\theta}$ to be a solution. For simplicity, let $\delta' = \delta \sqrt{m+k} = \frac{1}{100 \gamclwe} = \frac{1}{200 \sqrt{m}}$. Fix some $\vect{\theta} \in \S_{n,k}$. The random variable $\mat{A} \vect{\theta} \in \R^m$ (over the randomness of $\mat{A}$) is distributed according to $\N(\zero, k I_{m \times m})$. Therefore, $\norm{\mat{A} \vect{\theta}}_2^2 / k \sim \chi^2(m)$. As a result,
\[p := \Pr_{\mat{A}}\left[\norm{\mat{A} \vect{\theta}}_2^2 \leq (\delta')^2 \right] = \Pr_{Z\sim \chi^2(m)} \left[ Z \leq \frac{(\delta')^2}{k} \right]. \]
Using the CDF of the $\chi^2(m)$ distribution and the lower incomplete gamma function $\gamma(\cdot, \cdot)$, we have
\begin{align*}
p = \Pr_{Z\sim \chi^2(m)} \left[ Z \leq \frac{(\delta')^2}{k} \right] &= \frac{1}{\Gamma(m/2)} \cdot \gamma\left(\frac{m}{2}, \frac{(\delta')^2}{2k} \right)
\\&= \frac{1}{(m/2-1)!} \left(\frac{(\delta')^2}{2k} \right)^{m/2} e^{-(\delta')^2/2k} \cdot \\
& \quad \quad \quad \sum_{j = 0}^{\infty} \frac{\left(\frac{(\delta')^2}{2k} \right)^j}{(m/2)(m/2+1) \cdots (m/2 + j)}
\\&= (1 \pm o(1)) \cdot \frac{(\delta')^m}{(m/2-1)!(2k)^{m/2}} \cdot \frac{2}{m},
\\&= \left( \frac{\delta'}{\sqrt{mk}} \right)^{m(1 \pm o(1))}
\end{align*}
because in our setting, $(\delta')^2/2k = o(1)$, so $e^{-(\delta')^2/2k} = 1 - o(1)$ and the infinite sum is dominated by its first term, and because $\delta' = o(1)$ and $\sqrt{mk} = \omega(1)$. Therefore, by linearity of expectation, the expected number of solutions $X$ in $\S_{n,k}$ satisfies
\[ \E[X] = \left( \frac{n}{k} \right)^k \cdot p . \]
Setting $n$ so that
\[ p = \left( \frac{n}{k} \right)^{-k/10} \]
gives
\[ \E[X] = \left( \frac{n}{k} \right)^{9k/10}, \;\; n = k \cdot 2^{(10 \pm o(1))m/k \cdot \log(\sqrt{mk}/\delta')} = k \cdot 2^{\Theta( m/k \cdot \log(\sqrt{mk} \cdot \gamclwe))}. \]
We turn this expected value bound into a $1-o(1)$ bound on the existence of a sparse solution in %
\ifnum\llncs=0
Appendix~\ref{sec:clwe-proof-completeness}.
\else
the full version~\cite{gupte2024sparse}.
\fi
\footnote{We greatly thank Kiril Bangachev for turning the expected value statement into a high probability statement.}
\end{proof}

\ifnum\llncs=1
\begin{credits}
\subsubsection{\ackname} We greatly thank Kiril Bangachev for the proof of completeness in Theorem~\ref{thm:reduction-from-clwe} (in 
\ifnum\llncs=0
Appendix~\ref{sec:clwe-proof-completeness}
\else
the full version \cite{gupte2024sparse}\fi), and we thank Huck Bennett and Noah Stephens-Davidowitz for answering our questions about the $\BinBDD$ problem. We would also like to thank Stefan Tiegel for useful discussions about the performance of Lasso in the noiseless setting. AG was supported by the Ida M.~Green MIT Office of Graduate Education Fellowship and by the grants of the third author. NV was supported by NSF fellowship DGE-2141064 and by the grants of the third author. VV was supported in part by DARPA under Agreement No. HR00112020023, NSF CNS-2154149, a grant from the MIT-IBM Watson AI, a grant from Analog Devices, a Microsoft Trustworthy AI grant,  a Thornton Family Faculty Research Innovation Fellowship from MIT and a Simons Investigator Award. Any opinions, findings and conclusions or recommendations expressed in this material are those of the author(s) and do not necessarily reflect the views of the United States Government or DARPA.

\subsubsection{\discintname}
The authors have no competing interests to declare that are relevant to the content of this article.
\end{credits}
\else
\paragraph{Acknowledgements.} We greatly thank Kiril Bangachev for the proof of completeness in Theorem~\ref{thm:reduction-from-clwe} (in Appendix~\ref{sec:clwe-proof-completeness}), and we thank Huck Bennett and Noah Stephens-Davidowitz for answering our questions about the $\BinBDD$ problem. We would also like to thank Stefan Tiegel for useful discussions about the performance of Lasso in the noiseless setting.
\fi

\bibliographystyle{alpha}
\bibliography{refs}

\ifnum\llncs=0
\appendix
\newpage
\section{Complexity of Binary Bounded Distance Decoding}

In this section, we describe what is known about the computational complexity of BDD and $\BinBDD_{d, \alpha}$.

\subsection{Babai's Rounding Off Algorithm}
We analyze the ``rounding off'' algorithm given by~\cite{babai1986lovasz} and bound its performance in terms of the condition number $\cond(\matB)$ of the lattice basis $\matB$.\footnote{For the other algorithm of \cite{babai1986lovasz}, the nearest (hyper)plane algorithm, success occurs when $\alpha$ is at most a quantity defined in terms of the norms of the Gram-Schmidt orthogonalization of the basis $\mat{B}$, which cannot be cleanly bounded in terms of $\cond(\mat{B})$.} Given a full-rank lattice basis $\mat{B} \in \R^{d \times d}$ and target $\vect{t} := \mat{B} \vect{z} + \vect{e}$ for $\vect{z} \in \Z^d$ and $\norm{\vect{e}}_2 \leq \alpha \lambda_1(\mat{B})$, the rounding off algorithm simply outputs $\round(\mat{B}^{-1} \vect{t}) \in \Z^{d}$.\footnote{Often, Babai's rounding algorithm is performed on an LLL-reduced (\cite{lenstra1982factoring}) basis, and the resulting analysis of Babai's algorithm depends on the properties of the LLL reduction. Here, we instead analyze how Babai's algorithm would perform on a general basis in terms of the conditioning of the given basis. Note that lattice basis reduction does not preserve $\BinBDD_{d, \alpha}$ solutions as $\lambdabin(\matB)$ is basis-dependent.} To argue correctness, it suffices to show
\[ \norm{\mat{B}^{-1} \vect{t} - \vect{z}}_{\infty} < \frac{1}{2}, \]
as this would guarantee $\round(\mat{B}^{-1} \vect{t}) = \vect{z}$.

\begin{lemma}\label{lemma:rounding-off}
For $\alpha < \sigma_{\min}(\mat{B}) / (2 \lambda_1(\mat{B}))$, Babai's rounding off algorithm solves $\BDD$. The same is true for $\BinBDD_{d, \alpha}$ for $\alpha < \sigma_{\min}(\mat{B}) / (2 \lambdabin(\mat{B}))$.
\end{lemma}

\begin{proof}
For the case of $\BDD$, we have
\[ \norm*{\mat{B}^{-1} \vect{t} - \vect{z}}_{\infty} = \norm*{ \mat{B}^{-1} \vect{e}}_{\infty} \leq \norm*{ \mat{B}^{-1} \vect{e}}_2 \leq \sigma_{\max}(\mat{B}^{-1}) \norm{\vect{e}}_2 = \frac{\norm{\vect{e}}_2}{\sigma_{\min}(\mat{B})} \leq \frac{\alpha \lambda_1(\mat{B})}{ \sigma_{\min}(\mat{B})} < \frac{1}{2}, \]
as desired. The proof follows very similarly for $\BinBDD_{d, \alpha}$.
\end{proof}

\begin{corollary}
    For $\alpha < 1/(4 \cond(\matB))$, Babai's rounding off algorithm solves $\BDD$ and $\BinBDD$.
\end{corollary}

\begin{proof}
    Since $\lambda_1(\mat{B}) \leq  \lambdabin(\mat{B}) \leq 2\sigma_{\max}(\mat{B})$ by Lemma~\ref{lemma:lambda-1-bound}, we have 
    \[ \alpha < \frac{1}{4 \cond(\matB)} = \frac{\sigma_{\min}(\matB)}{4 \sigma_{\max}(\matB)} \leq \min\left( \frac{\sigma_{\min}(\mat{B})}{2 \lambda_1(\mat{B})}, \frac{\sigma_{\min}(\mat{B})}{2 \lambdabin(\mat{B})}\right). \]
    Therefore, we can invoke Lemma~\ref{lemma:rounding-off}.
\end{proof}
As a result, this algorithm is a better algorithm for $\BinBDD$ than in Theorem~\ref{thm:better-lasso-gives-better-bdd} for $\beta \leq 1$ but could be worse for $\beta > 1$ (up to $\poly(d)$ factors).

\subsection{Binary BDD Algorithm of \texorpdfstring{\cite{DBLP:conf/crypto/KirchnerF15}}{Kirchner and Fouque}}

\cite{DBLP:conf/crypto/KirchnerF15} introduce a variant of BDD very similar to $\BinBDD$. By improving the BKW algorithm (\cite{DBLP:journals/jacm/BlumKW03}), they show the following:

\begin{theorem}[Corollary 2 of \cite{DBLP:conf/crypto/KirchnerF15}]\label{thm:kirchner-fouque-alg}
    Given $(\matB \in \R^{d \times d}, \vect{t} := \matB \vect{z} + \vect{e})$ with $\vect{z} \in \{-1, 1\}^d$ and $\norm{\vect{e}}_2 \leq \alpha \lambda_1(\matB)$, there is an algorithm outputting $\vect{z}$ in time $2^{O\left(\frac{d}{\log \log (1/\alpha)} \right)}$
    as long as $\alpha = o(1)$ and $\alpha = 1/2^{o(d / \log d)}$.
\end{theorem}
Their algorithm also generalizes to $\vect{z} \in \{-B, \cdots, B\}^d$ instead of just binary $\vect{z}$ (with some dependence on $B$). The main difference between their setting and our definition of $\BinBDD$ is that for them, $\alpha$ is defined in terms of $\lambda_1(\matB)$ instead of $\lambdabin(\matB)$. Since $\lambda_1(\matB) \leq \lambdabin(\matB)$, it is not clear whether their algorithm gives a $2^{O\left(\frac{d}{\log \log (1/\alpha)} \right)}$ time algorithm when given $\norm{\vect{e}}_2 \leq \alpha \lambdabin(\matB)$ instead of  $\norm{\vect{e}}_2 \leq \alpha \lambda_1(\matB)$.

\cite{DBLP:conf/crypto/KirchnerF15} also give reductions from subset sum in low densities \cite[Theorem 14]{DBLP:conf/crypto/KirchnerF15} and bounded-norm variants of UniqueSVP \cite[Theorem 9]{DBLP:conf/crypto/KirchnerF15} and GapSVP \cite[Theorem 11]{DBLP:conf/crypto/KirchnerF15} to this bounded-norm variant of BDD. As a result of these reductions and their algorithm for bounded-norm BDD (i.e., Theorem \ref{thm:kirchner-fouque-alg}), they show algorithms for all three of these problems.

\subsection{Reduction from BDD to Binary BDD}\label{sec:reduction-from-bdd-to-binbdd}
We briefly mention the existence of a reduction from (standard) BDD to $\BinBDD$, at the cost of making the basis $\matB$ for the $\BinBDD$ instance have $\sigma_{\min}(\matB) = 0$ and therefore undefined $\cond(\matB)$.
First, we note that there is a direct reduction from $\{0,1\}$-$\BinBDD$ (i.e., where solutions are in $\{0,1\}^d$) to $\BinBDD$ that preserves the basis $\matB$ by mapping the target $\vect{t}$ to $2\vect{t} - \mat{B} \ones$.

Now, to reduce from standard BDD to $\{0,1\}$-$\BinBDD$, we first reduce the target $\vect{t}$ to lie in the fundamental parallelepiped of $\mat{B}$, and then we apply LLL~\cite{lenstra1982factoring} to get a $q := 2^{O(d)}$ upper bound on the size of the resulting coefficients of the reduced solution $\vect{z} \in \{0, 1, \cdots, q-1\}^d$. For simplicity, we can round $q$ up to the nearest power of $2$. Now, we use the gadget matrix
\[\mat{G} := I_d \otimes (1, 2, \cdots, q/2) \in \Z^{d \times d \log_2(q)}\]
used ubiquitously throughout lattice based cryptography (e.g., \cite{DBLP:conf/eurocrypt/MicciancioP12}) to map $\vect{z} \in \{0, \cdots, q-1\}^d$ to solutions $\vect{z}' \in \{0,1\}^{d \log_2(q)}$. Explicitly, for all $\vect{z} \in \{0, \cdots, q-1\}^d$, there exists a unique $\vect{z}' \in \{0,1\}^{d \log_2(q)}$ such that
\[ \vect{z} = \mat{G} \vect{z}'. \]
To see this, note that for all $i \in \{0, \cdots, d-1\}$, the entries of $(z'_{i \cdot \log_2(q) +1}, \cdots, z'_{i \cdot \log_2(q) + \log_2(q)})$ form the binary expansion of $z_i \in \{0, \cdots, q-1\}$. Therefore, we can let $\mat{B}' = \mat{B} \mat{G}$ and write 
\[ \vect{t} = \mat{B} \vect{z} + \vect{e} = \mat{B} \mat{G} \vect{z}' + \vect{e} = \mat{B}' \vect{z}' + \vect{e}, \]
making $(\mat{B}', \vect{t})$ a $\{0,1\}$-$\BinBDD$ instance, preserving the error $\vect{e}$. Furthermore, the new dimension $d'$ satisfies $d' = d \log_2(q) = O(d^2)$.

\section{Proof of Completeness of Theorem~\ref{thm:reduction-from-clwe}}\label{sec:clwe-proof-completeness}

Recall that $X$ is the number of solutions in $\S_{n,k}$. For $\vect{v} \in \S_{n,k}$, let $X_{\vect{v}} = \mathbbm{1}[ \norm{\mat{A} \vect{v}}_2 \leq \delta']$, where we have $p := \E[X_{\vect{v}}]$ and $X = \sum_{\vect{v} \in \S_{n,k}} X_{\vect{v}}$. By Chebyshev's inequality, we have
\[ \Pr[X = 0] \leq \frac{\Var(X)}{\E[X]^2} = \frac{\Var(X)}{p^2 \left( n/k \right)^{2k}} = \frac{\Var(X)}{\left( n/k \right)^{1.8k}}. \]
We will show that
\[ \Var(X) = o\left( \left( \frac{n}{k} \right)^{1.8k} \right), \]
which would imply that $\Pr[X \geq 1]= 1- o(1)$, as desired. For $\vect{u}, \vect{v} \in \S_{n,k}$, let $I(\vect{u}, \vect{v}) := |\support(\vect{u}) \cap \support(\vect{v})|$ denote the number of coordinates where they both take on the value $1$.

To do this, we can expand variance into covariance, and group by $I(\cdot, \cdot)$ between $\vect{u}, \vect{v} \in \S_{n,k}$:
\begin{align*}
    \Var(X) = \Var \left( \sum_{\vect{v} \in \S_{n,k}}X_{\vect{v}} \right) =\sum_{\vect{u}, \vect{v} \in \S_{n,k}} \Cov( X_{\vect{u}}, X_{\vect{v}}) = \sum_{t = 0}^k \sum_{\substack{\vect{u}, \vect{v} \in \S_{n,k}\\ I(\vect{u}, \vect{v}) = t}} \Cov( X_{\vect{u}}, X_{\vect{v}}).
\end{align*}
We can split this into $3$ terms, $Z_0, Z_1, Z_2$, as follows:
\[ \Var(X) = \underbrace{ \sum_{\substack{\vect{u}, \vect{v} \in \S_{n,k}\\ I(\vect{u}, \vect{v}) = 0}} \Cov( X_{\vect{u}}, X_{\vect{v}})}_{Z_0} + \underbrace{\sum_{t = 1}^{k/4} \sum_{\substack{\vect{u}, \vect{v} \in \S_{n,k}\\ I(\vect{u}, \vect{v}) = t}} \Cov( X_{\vect{u}}, X_{\vect{v}})}_{Z_1}  + \underbrace{\sum_{t > k/4} \sum_{\substack{\vect{u}, \vect{v} \in \S_{n,k}\\ I(\vect{u}, \vect{v}) = t}} \Cov( X_{\vect{u}}, X_{\vect{v}})}_{Z_2}. \]
We bound each of these terms separately by $o\left( \left( \frac{n}{k} \right)^{1.8k} \right)$, which gives the desired result. Before doing so, note that for all $t \in \{0, \cdots, k\}$,
\begin{equation}\label{eqn:number-of-pairs-with-t-in-common}
\left|\{(\vect{u}, \vect{v}) \in \S_{n,k} : I(\vect{u}, \vect{v}) = t \}\right| = \binom{k}{t} \left( \frac{n}{k}\right)^t \left( \frac{n}{k}\right)^{k-t} \left( \frac{n}{k} - 1\right)^{k-t} \leq  \binom{k}{t} \left( \frac{n}{k}\right)^{2k-t}, 
\end{equation}
since one needs to specify the $t$ parts in common, the common values on those $t$ parts, and two distinct values on the $k-t$ other parts.
\\\\\noindent\textbf{Bounding $Z_0$}: For $I(\vect{u}, \vect{v}) = 0$, since $\vect{u}$ and $\vect{v}$ do not overlap, they depend on disjoint columns of $\mat{A}$. Therefore, $X_{\vect{u}}$ and $X_{\vect{v}}$ are independent, making $\Cov(X_{\vect{v}}, X_{\vect{u}}) = 0$. Summing over $\vect{u}, \vect{v}$ with $I(\vect{u}, \vect{v}) = 0$ gives $Z_0 = 0$.
\\\\\noindent \textbf{Bounding $Z_1$}: For this case when $t = I(\vect{u}, \vect{v}) \in [1, k/4]$, the analysis becomes significantly more technical. 
Let $\vect{v}_0 \in \{0,1\}^n$ be the ``common'' vector with support $\support(\vect{u}) \cap \support(\vect{v})$, let $\vect{v}_1 \in \{0,1\}^n$ be the vector with support $\support(\vect{v}) \setminus \support(\vect{u})$, and let $\vect{v}_2 \in \{0,1\}^n$ be the vector with support $\support(\vect{u}) \setminus \support(\vect{v})$. Note that $\norm{\vect{v}_0}_0 =t$ and $\norm{\vect{v}_1}_0 = \norm{\vect{v}_2}_0 = k - t$. We have $\vect{v} = \vect{v}_0 + \vect{v}_1$ and $\vect{u} = \vect{v}_0 + \vect{v}_2$.

Now, consider the random variables $\vect{w} = \mat{A} \vect{v}_0$, $\vect{x} = \mat{A} \vect{v}_1$, and $\vect{y} = \mat{A} \vect{v}_2$. Since the supports of $\vect{v}_i$ are disjoint, the random variables $\vect{w}, \vect{x}, \vect{y}$ are independent. Moreover, $\vect{w} \sim \N(\zero, t I_{m \times m})$ and $\vect{x}, \vect{y} \sim \N(\zero, (k-t) I_{m \times m})$. For an arbitrary vector $\vect{z} \in \R^m$ and $\delta > 0$, let $B_{\delta}(\vect{z})$ denote the ball of radius $\delta$ around $\vect{z}$ in $\R^m$. Let $\mu_{\sigma^2}(\vect{z})$ denote the density function of $\N(\zero, \sigma^2 I_{m \times m})$ at the point $\vect{z} \in \R^m$. Explicitly,
\[ \mu_{\sigma^2}(\vect{z}) = \frac{1}{\left(\sqrt{2\pi \sigma^2}\right)^m} \exp \left( - \norm{\vect{z}}_2^2/(2 \sigma^2) \right). \]
We begin by using the bound
\[ \Cov(X_{\vect{u}}, X_{\vect{v}}) = \E[X_{\vect{u}} X_{\vect{v}}] - \E[X_{\vect{u}}] \E[X_{\vect{v}}] \leq \E[X_{\vect{u}} X_{\vect{v}}]. \]
Then,
\begin{align*}
    \E[X_{\vect{u}} X_{\vect{v}}]= \Pr_{\matA}[ \norm{\matA \vect{u}}_2 \leq \delta' \land \norm{\matA \vect{v}}_2 \leq \delta' ]  &= \Pr_{\vect{w}, \vect{x}, \vect{y}}[ \norm{\vect{w} + \vect{x}}_2 \leq \delta' \land \norm{\vect{w} + \vect{y}}_2 \leq \delta' ]
    \\&= \Pr_{\vect{w}, \vect{x}, \vect{y}}[ \vect{x} \in B_{\delta'}(-\vect{w}) \land \vect{y} \in B_{\delta'}(-\vect{w}) ]
    \\&= \int_{\R^m} \left(\Pr_{\vect{x}}[\vect{x} \in B_{\delta'}(-\vect{w})] \right)^2 \mu_t(\vect{w}) d \vect{w}
    \\&=  \int_{\R^m}  \left( \int_{ B_{\delta'}(-\vect{w})} \mu_{k-t}(\vect{x}) d\vect{x} \right)^2 \mu_t(\vect{w}) d \vect{w}.
\end{align*}
We split this integral up into two terms $Y_0$ and $Y_1$, depending on the norm of $\vect{w}$: either $\norm{\vect{w}}_2 \leq k/\delta'$ or not. Explicitly, we have
\begin{align} \Cov(X_{\vect{u}}, X_{\vect{v}}) &\leq \int_{\R^m}  \left( \int_{ B_{\delta'}(-\vect{w})} \mu_{k-t}(\vect{x}) d\vect{x} \right)^2 \mu_t(\vect{w}) d \vect{w}
\\&= \underbrace{\int_{B_{k/\delta'}(\zero)}  \left( \int_{ B_{\delta'}(-\vect{w})} \mu_{k-t}(\vect{x}) d\vect{x} \right)^2 \mu_t(\vect{w}) d \vect{w}}_{Y_0} 
\\&+\underbrace{\int_{\R^m \setminus B_{k/\delta'}(\zero)}  \left( \int_{ B_{\delta'}(-\vect{w})} \mu_{k-t}(\vect{x}) d\vect{x} \right)^2 \mu_t(\vect{w}) d \vect{w}}_{Y_1}.\label{eq:cov-Y0-Y1}
\end{align}

\noindent \textbf{(Subcase 1) Bounding $Y_0$}: For $\vect{w} \in B_{k/\delta'}(\zero)$, we have $\norm{\vect{w}}_2 \leq k/\delta'$. The point is that in this case, the density $\mu_{k-t}(\vect{x})$ is almost uniform for $\vect{x} \in B_{\delta'}(-\vect{w})$. To see this, consider two arbitrary points $-\vect{w} - \vect{a}, -\vect{w} - \vect{b} \in B_{\delta'}(-\vect{w})$, where $\norm{\vect{a}}, \norm{\vect{b}} \leq \delta'$. The ratio of the densities is given by
\begin{align*}
\frac{\mu_{k-t}(-\vect{w} - \vect{a})}{\mu_{k-t}(- \vect{w} - \vect{b})} &= \frac{\exp \left( - \norm{\vect{w} + \vect{a}}_2^2/(2 (k-t)) \right)}{\exp \left( - \norm{ \vect{w} + \vect{b}}_2^2/(2 (k-t)) \right)}
\\&= \exp \left(  \frac{-\inner{\vect{w} + \vect{a}, \vect{w} + \vect{a}} +\inner{\vect{w} + \vect{b}, \vect{w} + \vect{b}}}{2(k-t)} \right)
\\&= \exp \left(  \frac{- \norm{\vect{a}}_2^2 + \norm{\vect{b}}_2^2 + 2 \inner{\vect{w}, \vect{b} - \vect{a}}}{2(k-t)} \right).
\end{align*}
By Cauchy-Schwartz and the triangle inequality, we have 
\[|2\inner{\vect{w}, \vect{b} - \vect{a}}| \leq 2 \norm{\vect{w}} \norm{\vect{b} - \vect{a}} \leq 2 \norm{\vect{w}} (\norm{\vect{b}} + \norm{\vect{a}}) \leq 4 \frac{k}{\delta'} \delta' = 4 k.\]
For the other two terms in the numerator, we have $\norm{\vect{a}}^2,\norm{\vect{b}}^2 \leq (\delta')^2 = o(1)$. Therefore, the absolute value of the numerator in the exponent is at most $O(k)$. Since $t \leq k/4$, we know $k-t \geq 3k/4$, making the exponent at most $O(k)/\Omega(k) = O(1)$. Therefore, there is a constant $C = O(1)$ such that
\[ \frac{1}{C} \leq \frac{\mu_{k-t}(-\vect{w} - \vect{a})}{\mu_{k-t}(- \vect{w} - \vect{b})} \leq C.\]
Therefore, we get
\begin{align*}
    Y_0 &= \int_{B_{k/\delta'}(\zero)}  \left( \int_{ B_{\delta'}(-\vect{w})} \mu_{k-t}(\vect{x}) d\vect{x} \right)^2 \mu_t(\vect{w}) d \vect{w}
    \\&\leq C^2 \cdot \int_{B_{k/\delta'}(\zero)}  \left( \int_{ B_{\delta'}(-\vect{w})} \mu_{k-t}(-\vect{w}) d\vect{x} \right)^2 \mu_t(\vect{w}) d \vect{w}
    \\&= C^2 \int_{B_{k/\delta'}(\zero)}  \Vol( B_{\delta'}(-\vect{w}))^2  \cdot  \mu_{k-t}(-\vect{w})^2 \mu_t(\vect{w}) d \vect{w}
    \\&= C^2 \Vol( B_{\delta'}(\zero))^2 \int_{B_{k/\delta'}(\zero)}   \mu_{k-t}(-\vect{w})^2 \mu_t(\vect{w}) d \vect{w}
    \\&\leq C^2 \Vol( B_{\delta'}(\zero))^2 \int_{\R^m}   \mu_{k-t}(-\vect{w})^2 \mu_t(\vect{w}) d \vect{w}.
\end{align*}
Continuing the inequality, we have
\begin{align}
    Y_0 &\leq C^2 \Vol( B_{\delta'}(\zero))^2 \int_{\R^m}   \mu_{k-t}(-\vect{w})^2 \mu_t(\vect{w}) d \vect{w}
    \\&= C^2 \Vol( B_{\delta'}(\zero))^2 \frac{1}{ \left(\sqrt{2 \pi (k-t)} \right)^{2m} \left(\sqrt{2 \pi t} \right)^{m} } \int_{\R^m} \exp \left( - \frac{2\norm{\vect{w}}^2}{2(k-t)} - \frac{\norm{\vect{w}}^2}{2t} \right)  d \vect{w}
    \\&= C^2 \Vol( B_{\delta'}(\zero))^2 \frac{1}{ \left(\sqrt{2 \pi (k-t)} \right)^{2m} \left(\sqrt{2 \pi t} \right)^{m} } \int_{\R^m} \exp \left( - \frac{k+t}{2t(k-t)} \norm{\vect{w}}^2 \right)  d \vect{w}
    \\&= C^2 \Vol( B_{\delta'}(\zero))^2 \frac{ \left( \sqrt{2 \pi \cdot \frac{t(k-t)}{k+t} }\right)^m }{ \left(\sqrt{2 \pi (k-t)} \right)^{2m} \left(\sqrt{2 \pi t} \right)^{m} }
    \\&= C^2 \Vol( B_{\delta'}(\zero))^2 \left( (2 \pi)^2 (k-t)(k+t) \right)^{-m/2}
    \\&\leq C^2 \Vol( B_{\delta'}(\zero))^2 \left( (2 \pi)^2 k^2\right)^{-m/2} \left( 1 + \frac{2t^2}{k^2}\right)^{m/2}
    \\&\leq C^2 \Vol( B_{\delta'}(\zero))^2 \left( (2 \pi)^2 k^2\right)^{-m/2} \exp \left(t^2 m/k^2 \right)\label{eqn:Y0-bound},
\end{align}
where the penultimate inequality uses the fact that $t \leq k/4$. We can similarly lower bound $p$ as follows:
\begin{align}
    p = \E[X_{\vect{v}}] &= \int_{\R^m} \int_{B_{\delta'}(-\vect{w})} \mu_{k-t}(\vect{x}) d \vect{x} \mu_t(\vect{w}) d\vect{w}
    \\&\geq \int_{B_{k/\delta'}(\zero)} \int_{B_{\delta'}(-\vect{w})} \mu_{k-t}(\vect{x}) d \vect{x} \mu_t(\vect{w}) d\vect{w}
    \\&\geq \frac{1}{C} \int_{B_{k/\delta'}(\zero)} \int_{B_{\delta'}(-\vect{w})} \mu_{k-t}(-\vect{w}) d \vect{x} \mu_t(\vect{w}) d\vect{w}
    \\&= \frac{1}{C} \Vol( B_{\delta'}(\zero))  \int_{B_{k/\delta'}(\zero)} \mu_{k-t}(-\vect{w})\mu_t(\vect{w}) d\vect{w}
    \\&= \frac{1}{C} \Vol( B_{\delta'}(\zero))  \left( \int_{\R^m} \mu_{k-t}(-\vect{w})\mu_t(\vect{w}) d\vect{w} - \int_{\R^m \setminus B_{k/\delta'}(\zero)} \mu_{k-t}(-\vect{w})\mu_t(\vect{w}) d\vect{w} \right).\label{eq:p-bound-integrals}
\end{align}
For $\vect{w} \in \R^m \setminus B_{k/\delta'}(\zero)$, we have $\norm{\vect{w}} \geq k/\delta'$. This implies (as a weak bound)
\begin{align*}
\mu_{k-t}(-\vect{w}) &= \frac{1}{\left(\sqrt{2\pi (k-t)}\right)^m} \exp \left( - \norm{\vect{w}}_2^2/(2 (k-t))\right)
\\&\leq \exp(-k/(2(\delta')^2) = \exp(-m^{1 + \Omega(1)}) = o(p).
\end{align*}
Since $ \Vol( B_{\delta'}(\zero)) < 1$ for $\delta' = o(1)$, it follows that
\begin{align*} \frac{1}{C} \Vol( B_{\delta'}(\zero))  \int_{\R^m \setminus B_{k/\delta'}(\zero)} \mu_{k-t}(-\vect{w})\mu_t(\vect{w}) d\vect{w} \leq o(p) \int_{\R^m \setminus B_{k/\delta'}(\zero)} \mu_t(\vect{w}) d\vect{w} \leq o(p).
\end{align*}
Therefore, by adding to~\eqref{eq:p-bound-integrals} and dividing by $1 + o(1)$,
\begin{align*}
    p &\geq \frac{1}{2C} \Vol( B_{\delta'}(\zero)) \int_{\R^m} \mu_{k-t}(-\vect{w})\mu_t(\vect{w}) d\vect{w}
    \\&= \frac{1}{2C} \Vol( B_{\delta'}(\zero))  \frac{1}{ \left(\sqrt{2 \pi (k-t)} \right)^{m} \left(\sqrt{2 \pi t} \right)^{m} } \int_{\R^m} \exp \left( \frac{-\norm{\vect{w}}^2}{2(k-t)} - \frac{\norm{\vect{w}}^2}{2t} \right) d\vect{w}
    \\&= \frac{1}{2C} \Vol( B_{\delta'}(\zero))  \frac{1}{ \left(\sqrt{2 \pi (k-t)} \right)^{m} \left(\sqrt{2 \pi t} \right)^{m} } \int_{\R^m} \exp \left(- \frac{k}{2t(k-t)} \norm{\vect{w}}^2 \right) d\vect{w} 
    \\&= \frac{1}{2C} \Vol( B_{\delta'}(\zero)) (2 \pi k)^{-m/2}.
\end{align*}
Using this lower bound on $p$, we can use~\eqref{eqn:Y0-bound} to see that
\[ Y_0 \leq C^2 \Vol( B_{\delta'}(\zero))^2 \left( (2 \pi)^2 k^2\right)^{-m/2} \exp \left(t^2 m/k^2 \right) \leq C_1 p^2 \exp \left(t^2 m/k^2 \right), \]
for some constant $C_1 = O(1)$.
\\\\\noindent \textbf{(Subcase 2) Bounding $Y_1$}: For $\vect{w} \in \R^m \setminus B_{k/\delta'}(\zero)$, we have $\norm{\vect{w}}_2 \geq k/\delta'$. For any $\vect{x} \in B_{\delta'}(-\vect{w})$, it follows by the triangle inequality that 
\[ \norm{\vect{x}} \geq \norm{-\vect{w}} - \delta' \geq k/\delta' - \delta' \geq k/(2 \delta'), \]
since $\delta' = o(1)$. Therefore, since $k-t \leq k$, for any $\vect{x} \in B_{\delta'}(-\vect{w})$, we have
\begin{align*} \mu_{k-t}(\vect{x}) &= \frac{1}{\left(\sqrt{2\pi (k-t)}\right)^m} \exp \left( - \norm{\vect{x}}_2^2/(2 (k-t)) \right) 
\\&\leq  \frac{1}{\left(\sqrt{2\pi (k-t)}\right)^m} \exp \left( - k/(8 (\delta')^2) \right)
\\&\leq \exp \left( - m^{1 + \Omega(1)} \right).
\end{align*}
As a weak bound, since $\Vol(B_{\delta'}(-\vect{w})) \leq 1$ (as $\delta' = o(1)$), this implies for all $\vect{x} \in B_{\delta'}(-\vect{w})$,
\[ \left( \int_{ B_{\delta'}(-\vect{w})} \mu_{k-t}(\vect{x}) d\vect{x} \right)^2 \leq \exp \left( - m^{1 + \Omega(1)} \right), \]
and therefore,
\begin{align*}
Y_1 &= \int_{\R^m \setminus B_{k/\delta'}(\zero)}  \left( \int_{ B_{\delta'}(-\vect{w})} \mu_{k-t}(\vect{x}) d\vect{x} \right)^2 \mu_t(\vect{w}) d \vect{w}
\\&\leq   \exp \left( - m^{1 + \Omega(1)} \right) \int_{\R^m \setminus B_{k/\delta'}(\zero)} \mu_t(\vect{w}) d \vect{w} \leq \exp \left( - m^{1 + \Omega(1)} \right).
\end{align*}
Since $p^2 = \exp \left(-\Theta(m \log m) \right) = \omega\left( \exp \left( - m^{1 + \Omega(1)} \right) \right)$, it follows that
\[ Y_1 = o(p^2). \]
\\\\\noindent Combining the bounds on $Y_0$ and $Y_1$ via Eq.~\eqref{eq:cov-Y0-Y1}, there is a constant $C_2 = O(1)$ for which
\[ \Cov(X_{\vect{u}}, X_{\vect{v}}) \leq Y_0 + Y_1 \leq C_1 p^2 \exp \left(t^2 m/k^2 \right) + o(p^2) \leq C_2 p^2 \exp \left(t^2 m/k^2 \right). \]
Using the definition of $Z_1$, \eqref{eqn:number-of-pairs-with-t-in-common}, and $t \leq k/4$, we have
\begin{align*}
    Z_1 = \sum_{t = 1}^{k/4} \sum_{\substack{\vect{u}, \vect{v} \in \S_{n,k}\\ I(\vect{u}, \vect{v}) = t}} \Cov( X_{\vect{u}}, X_{\vect{v}}) &\leq  \sum_{t = 1}^{k/4} \sum_{\substack{\vect{u}, \vect{v} \in \S_{n,k}\\ I(\vect{u}, \vect{v}) = t}} C_2 p^2 \exp \left(t^2 m/k^2 \right)
    \\&\leq C_2 p^2  \sum_{t = 1}^{k/4}  \binom{k}{t} \left( \frac{n}{k} \right)^{2k-t}\exp \left(t^2 m/k^2 \right)
    \\&\leq C_2 p^2\left( \frac{n}{k} \right)^{2k} \sum_{t = 1}^{k/4}  \binom{k}{t} \left( \frac{n}{k} \right)^{-t}\exp \left(t m/(4k) \right)
    \\&\leq C_2 p^2\left( \frac{n}{k} \right)^{2k} \sum_{t = 1}^{k/4} \left( \frac{k^2 \exp(m/(4k))}{n} \right)^t
    \\&= C_2 \left( \frac{n}{k} \right)^{1.8k} \sum_{t = 1}^{k/4} \left( \frac{k^2 \exp(m/(4k))}{n} \right)^t.
\end{align*}
Since this geometric sum starts at index $t=1$, to show $Z_1 = o((n/k)^{1.8k})$, it suffices to show that the base, $\frac{k^2\exp(m/(4k))}{n}$, is $o(1)$. To see this, recall that
\[n = k \cdot \exp \left( \Theta(m/k \cdot \log m) \right),\]
and $k = m^{\epsilon}$ for $\epsilon \in (0,1)$, making
\begin{align*}
\frac{k^2\exp(m/(4k))}{n} &= \frac{k \exp(m/(4k))}{\exp \left( \Theta(m/k \cdot \log m) \right)}
\\&\leq m \exp \left( -\Theta(m/k \cdot \log m) \right) 
\\&\leq m \exp \left( -m^{\Omega(1)} \log m\right)
\\& = o(1),
\end{align*}
as desired.

\noindent \textbf{Bounding $Z_2$}: For this case when $t > k/4$, we can use the (very) weak bound that
\[ \Cov(X_{\vect{u}}, X_{\vect{v}}) = \E[X_{\vect{u}} X_{\vect{v}}] - \E[X_{\vect{u}}] \E[X_{\vect{v}}] \leq \E[X_{\vect{u}} X_{\vect{v}}] \leq 1. \]
Using~\eqref{eqn:number-of-pairs-with-t-in-common}, we get
\begin{align*}
    Z_2 = \sum_{t > k/4} \sum_{\substack{\vect{u}, \vect{v} \in \S_{n,k}\\ I(\vect{u}, \vect{v}) = t}} \Cov( X_{\vect{u}}, X_{\vect{v}}) &\leq \sum_{t > k/4} \binom{k}{t} \left( \frac{n}{k}\right)^{2k-t}
    \\&\leq \sum_{t > k/4} \binom{k}{t} \left( \frac{n}{k}\right)^{1.75k}
    \\&\leq \left( \frac{n}{k}\right)^{1.75k} \cdot 2^k
    \\&= o\left(\left( \frac{n}{k}\right)^{1.8 k}  \right),
\end{align*}
where the last line uses the fact that $n/k = \omega(1)$.

\section{Reduction to \texorpdfstring{$k$-$\SLR$}{k-SLR} with Independent Noise}\label{section:independent-noise}

The noise $\vect{w} = \vect{y} - \mat{X} \vect{\theta}^*$ in our instance as described in Theorem~\ref{thm:main-bin-bdd-to-slr} depends on $\mat{R} \vect{e}$, whose entries are marginally Gaussian with variance $\norm{\vect{e}}_2^2$, but not independent from the samples in $\mat{X}$. By adding sufficient flooding noise $\vect{\xi}$ to the vector $\vect{y}$, and by starting with $\BinBDD_{d, \alpha}$ instances for which $\alpha = 1/\poly(m)$, we can obtain a distribution over the noise that is at most $1/2$-away in total variation distance from being independent Gaussian.

\begin{lemma}\label{lemma:independent-noise}
Let $d, m, n, k, m_1$ be integers such that $k$ divides $d$, $m \ge \Omega(d)$ and $n = k \cdot 2^{d/k}$ and $m_1 = m-k$. Let $\alpha = o(1/(m \sqrt{\log m}))$ be a real number. Then, there is a $\poly(n, m)$-time reduction from $\BinBDD_{d, \alpha}$ to $k$-$\SLR$ in dimension $n$ with $m$ samples that succeeds with probability at
least $9/10$. If the $\BinBDD_{d, \alpha}$ instance has lattice basis $\matB$, the resulting $k$-$\SLR$ instance $(\matX, \vecy)$ has parameter $\delta = \Theta(\lambda_{1, \bin}(\matB))$ and design matrix $\mat{X}$ as structured in Theorem~\ref{thm:main-bin-bdd-to-slr}. Further, there is a $k$-sparse vector $\vect{\theta}^*$ such that
\begin{align*}
    \matX \vect{\theta}^* + \vect{w} = \vect{y}
\end{align*}
such that noise $\vect{w}$ is drawn from a distribution that is at most $1/2$-away from $\N(0, \sigma^2)^{\otimes m}$ in total variation distance, where
\begin{align*}
    \sigma = \Theta\left(\alpha \cdot \lambda_{1, \bin}(\matB) \cdot m \sqrt{\log m}\right).
\end{align*}
\end{lemma}

\begin{proof}[Sketch.] Let $(\mat{B}, \vect{t} := \mat{B} \vect{z} + \vect{e})$ be our $\BinBDD_{d, \alpha}$ instance, where $\mat{B} \in \R^{d \times d}$, $\vect{z} \in \{-1, 1\}^d$ and $\|\vect{e}\|_2 \le \alpha \cdot \lambda_{1, \bin}(\mat{B})$. Let $\mat{R} \in \R^{m_1 \times d}$ be a random matrix where each entry is drawn i.i.d.~ from $\mathcal{N}(0, 1)$. We define the design matrix $\mat{X}$ and target vector $\vect{y}$ for $k$-$\SLR$ similar to the definition in Theorem~\ref{thm:main-bin-bdd-to-slr}. The main difference is that to achieve Gaussian noise that is independent of the samples $\mat{X}$, we add noise vector $\vect{\xi} = \begin{pmatrix} \vect{\xi}_1\\ \vect{\xi}_2 \end{pmatrix}$ to the vector $\vect{y}$, to get the instance $(\mat{X}, \vect{y})$ where
\begin{align*}
    \mat{X} =
    \begin{pmatrix}
        \mat{X}_1\\
        \mat{X}_2
    \end{pmatrix} = 
    \begin{pmatrix}
        \mat{R} \mat{B} \mat{G}_{\sparse}\\
        \gamma \mat{G}_{\partite}
    \end{pmatrix}, \quad \text{ and } \quad
    \vect{y} = 
    \begin{pmatrix}
        \vect{y}_1\\
        \vect{y}_2
    \end{pmatrix}
    =
    \begin{pmatrix}
        \mat{R} \vect{t} + \vect{\xi}_1\\
        \gamma \vect{1} + \vect{\xi}_2
    \end{pmatrix},
\end{align*}
where each entry in $\vect{\xi}_1, \vect{\xi}_2$ is sampled i.i.d from $\N(0, \sigma^2)$. 
As before in Theorem~\ref{thm:main-bin-bdd-to-slr}, we get an approximation $\widehat{\lambda}_1$ for $\lambda_{1, \bin}(\mat{B})$. (Iterating over different guesses will take at most polynomial time in the size of the problem description.) In the remainder of the proof, we will assume we know a value $\widehat{\lambda}_1$ such that $\widehat{\lambda}_1 \in [\lambda_{1, \bin}(\mat{B}), 2 \lambda_{1, \bin}(\mat{B}))$.

We set the parameters $\delta, \sigma, \gamma$ such that
\begin{align*}
    \delta^2 &= \frac{3 m_1 \cdot \widehat{\lambda}_1^2}{100m} = \Theta(\lambda_{1, \bin}(\mat{B})^2),\\ \sigma &= \Theta(\alpha \cdot \widehat{\lambda}_1 \cdot m_1 \sqrt{\log (m_1)}), \text{ and}\\
    \gamma &= \Omega\left( \frac{\sigma_{\max}(\mat{B})}{\widehat{\lambda}_1} \cdot \sqrt{\frac{d}{k}} \cdot (\delta \sqrt{m} + \sigma \sqrt{k}) \right).
\end{align*}
Suppose there is an algorithm that returns a $k$-sparse $\widehat{\vect{\theta}}$ when run on instance $(\mat{X}, \vect{y})$ with parameter $\delta$ such that
\begin{align}\label{equation:k-slr-solution-independent-noise}
    \frac{\|\mat{X} \widehat{\vect{\theta}} - \vect{y} \|_2^2}{m} \le \delta^2.
\end{align}
Then, we round $\widehat{\vect{\theta}}$ to the nearest integer entrywise to get $\mathsf{round}(\widehat{\vect{\theta}}) \in \Z^n$, and we output $\mat{G}_{\sparse} \mathsf{round}(\widehat{\vect{\theta}}) \in \Z^d$ as the $\BinBDD_{d, \alpha}$ solution.

First, we show that this reduction is complete. For completeness to hold, it suffices to choose $\delta, \sigma$ such that the following holds.
\begin{align*}
    \frac{1}{m} \norm{\mat{X} \vect{\theta}^* - \vect{y}}_2^2 &= \frac{1}{m} \left( \norm{\mat{R} \vect{e} + \vect{\xi}_1}_2^2 + \norm{\vect{\xi}_2}_2^2 \right) 
    \le  O \left( \frac{\norm{\mat{R} \vect{e}}_2^2 + \norm{\vect{\xi}_1}_2^2 + \norm{\vect{\xi}_2}_2^2}{m} \right) \le O \left(\frac{m_1 \norm{\vect{e}}_2^2 + m \sigma^2}{m}\right)\\
    &\le O(\norm{\vect{e}}_2^2 + \sigma^2) \le \delta^2.
\end{align*}

Now, we proceed to prove that the reduction is sound. For the sake of contradiction, suppose that $\mat{G}_\sparse \mathsf{round}(\widehat{\vect{\theta}}) = \vect{z}'$ for some $\vect{z}' \neq \vect{z}$ such that Equation~\eqref{equation:k-slr-solution-independent-noise} holds. Since Equation~\eqref{equation:k-slr-solution-independent-noise} holes, in particular, we know that $\|\mat{X}_2 - \vect{y}_2\|_2^2 = \norm{\gamma \mat{G}_\partite \thetahat  - \gamma \vect{1} - \vect{\xi}_2}_2^2 \le m \delta^2$. Since $\|\vect{\xi}_2\|_2 \le O(\sigma \sqrt{k})$ with high probability, we set $\gamma$ such that
\begin{align*}
    \norm{\mat{G}_\partite \thetahat - \vect{1}}_2 \le \frac{\delta \sqrt{m} + \|\vect{\xi}_2\|_2}{\gamma} < 1/2.
\end{align*}
Since $\thetahat$ is $k$-sparse, this means that it cannot have any part with all zero entries. So each part of $\thetahat$ must have exactly one non-zero entry. By the choice of $\gamma$, each of these non-zero entries must be in the interval $(1/2, 3/2)$, and must round to $1$. Therefore $\mathsf{round}(\thetahat) \in \mathcal{S}_{n,k}$ and $\mathsf{round}(\thetahat)$ and $\thetahat$ have the same support. This implies that $\vect{z}' = \mat{G}_\sparse \mathsf{round}(\thetahat) \in \{-1, +1\}^d$ and $\mat{G}_\partite \mathsf{round}(\thetahat) = \mathbf{1}$ by Lemma~\ref{lemma:G-forces-partite-binary}. Since we can restrict to the support of $\thetahat$, we have
\begin{align*}
    \norm{\thetahat} - \round(\thetahat)_2 = \norm{\mat{G}_\partite \thetahat - \mat{G}_\partite \round(\mat{G})} = \norm{\matG_\partite \thetahat - \mathbf{1}}_2 \le \frac{\delta \sqrt{m} + \norm{\vect{\xi}_2}_2}{\gamma}.
\end{align*}
Since $\thetahat - \round(\thetahat)$ is also $k$-sparse and partite, it follows that
\begin{align*}
    \norm{\matG_\sparse(\thetahat - \round(\thetahat))}_2 \le \sqrt{\frac{d}{k}} \cdot \frac{\delta \sqrt{m} + \norm{\vect{\xi}_2}_2}{\gamma},
\end{align*}
as each column of any block of $\mat{G}_\sparse$ has $\ell_2$ norm exactly $\sqrt{d/k}$. Multiplying on the left by $\matB$, we have
\begin{align*}
\norm{\mat{B} \mat{G}_\sparse (\thetahat - \round(\thetahat))}_2 &= \norm{\mat{B} \mat{G}_{\sparse} - \mat{B} \vect{z}'}_2\\
&\le \sigma_{\max}(\mat{B}) \cdot \sqrt{\frac{d}{k}} \cdot \frac{\delta \sqrt{m} + \norm{\vect{\xi}_2}_2}{\gamma}.
\end{align*}
On the other hand, by definition of $\lambda_{1, \bin}(\mat{B})$, since $\vect{z} \neq \vect{z}' \in \{-1, 1\}^d$, we know that $\norm{\mat{B} \vect{z} - \mat{B} \vect{z}'}_2 \ge \lambda_{1, \bin}(\mat{B})$. Therefore, by triangle inequality, and by our choice of $\gamma$, we have
\begin{align*}
    \norm{\mat{B} \mat{G}_\sparse \thetahat - \mat{B} \vect{z}}_2 &\ge \norm{\mat{B} \vect{z} - \mat{B} \vect{z}'}_2 - \norm{\mat{B} \mat{G}_\sparse \thetahat - \mat{B} \vect{z}'}_2 \ge \frac{49}{50} \lambda_{1, \bin}(\mat{B}).
\end{align*}
By incorporating the error $\vect{e}$ and multipling on the left by $\mat{R}$, we have
\begin{align*}
    \norm{\mat{X} \thetahat - \vect{y}}_2 &\ge \norm{\mat{R} 
    \mat{B} \mat{G}_\sparse \thetahat - \mat{R} (\mat{B} \vect{z} + \vect{e}) - \vect{\xi}_1}_2 \\
    &\ge \sigma_{\min}(\mat{R}) \cdot \left( \norm{\mat{B} \mat{G}_\sparse \thetahat - \mat{B} \vect{z}}_2 - \norm{\vect{e}}_2 - \norm{\vect{\xi}_1}_2 \right)\\
    &\ge \frac{\sqrt{m_1}}{2} \cdot \left( \frac{49}{50} \lambda_{1, \bin}(\mat{B}) - \alpha \lambda_{1, \bin}(\mat{B}) - c_1 \sigma \sqrt{m_1} \right).
\end{align*}
The last inequality follows because $\norm{\vect{\xi}_1}_2 \le c_1 \sigma \sqrt{m_1}$ for some constant $c_1 > 0$ with probability at least $1 - e^{-\Omega(m_1)}$. Also, by Lemma~\ref{lemma:gaussian-singular-values}, we know that $\sigma_{\min}(\mat{R}) \ge \sqrt{m_1}/2$ with probability at least $1 - e^{-\Omega(m_1)}$.
By Equation~\eqref{equation:k-slr-solution-independent-noise}, we get a contradiction as long as
\begin{align*}
    \delta \sqrt{m} < \sqrt{m_1} \cdot \left( \frac{49}{50} \lambda_{1, \bin}(\mat{B}) - \alpha \cdot \lambda_{1, \bin}(\mat{B}) - c_1 \sigma \sqrt{m_1} \right).
\end{align*}
On the other hand, we know
\begin{align*}
    \delta \sqrt{m} = \frac{\widehat{\lambda}_1 \sqrt{3 m_1}}{10\sqrt{m}} \cdot \sqrt{m} \le \frac{1}{5} \cdot \lambda_{1, \bin}(\mat{B}) \cdot \sqrt{3m_1}.
\end{align*}
Combining these inequalities we get a contradiction if $\alpha < c_2$ and $\sigma < c_3 \cdot \lambda_{1, \bin}(\mat{B})/\sqrt{m_1}$, for some constants $c_2, c_3 > 0$. Since $\alpha = o(1/(m \sqrt{\log m}))$, and by the choice of $\sigma$, these conditions are met.

Finally, we want the distributions of $\mat{R} \vect{e} + \vect{\xi}_1$ and that of $\vect{\xi}_1$ to have total variation distance at most $1/2$. With probability at least $99/100$, $\norm{\mat{R} \vect{e}}_{\infty} \le O\left(\norm{\vect{e}}_2 \cdot \sqrt{\log(m_1)} \right)$ by standard Gaussian tail bounds and a union bound. By Lemma~\ref{lemma:flooding}, and applying a union bound, as long as $\sigma \ge \Omega\left(\norm{\vect{e}}_2 \cdot m_1 \sqrt{\log(m_1)} \right)$, the distributions of $\mat{R} \vect{e} + \vect{\xi}_1$ and of $\vect{\xi}_1$ are at most $1/2$-apart in total variation distance. It suffices that $\sigma \ge \Omega(\alpha \cdot \lambda_{1, \bin}(\mat{B}) \cdot m_1 \sqrt{\log(m_1)})$.

Since $m_1 = \Theta(m)$, choosing $\sigma = \Theta \left(\alpha \cdot \lambdabin(\mat{B}) \cdot m \sqrt{\log(m)}\right)$, $\delta = \Theta(\lambdabin)$ and $\alpha = o(1/(m \sqrt{\log m}))$ satisfies all of these constraints. This completes the proof.
\end{proof}

\begin{lemma}\label{lemma:flooding}
Let $|x| \le b$ be a real number and let $Y \sim \mathcal{N}(0, \sigma^2)$ be a random variable. If $\sigma \ge \frac{b}{2\epsilon}$, then the distributions of $x + Y$ and $Y$ have total variation distance at most $\epsilon$.
\end{lemma}
\begin{proof}
The random variable $x + Y$ is distributed as $\mathcal{N}(x, \sigma^2)$, so the KL divergence between the distributions of $x+Y$ and $Y$ is $\KL(x+Y \| Y) = \frac{x^2}{2\sigma^2} \le \frac{b^2}{2 \sigma^2}$. By Pinsker's inequality, the total variation distance between the two distributions is $\TV(x + Y, Y) \le \sqrt{\frac{1}{2} \KL(x+Y \| Y)} \le \frac{b}{2 \sigma} \le \epsilon$.
\end{proof}

\fi

\end{document}